\documentclass[11pt]{article}

\newif\ifappendix
\appendixtrue

\newcommand{\refappendix}[1]{\ifappendix
 Appendix~\ref{#1}\xspace
\else
 the supplementary material\xspace
\fi}

\usepackage[sc]{mathpazo}

\usepackage[margin=1in]{geometry}
\usepackage[english]{babel}
\usepackage[utf8x]{inputenc}
\usepackage[compact]{titlesec}

\usepackage{cmap}
\usepackage[T1]{fontenc}
\usepackage{bm}
\usepackage{amsmath}
\usepackage{amsfonts}
\usepackage{amssymb}
\usepackage{amsbsy}
\usepackage{amsthm}
\usepackage{dsfont}

\usepackage{mathtools}
\usepackage{xspace}

\usepackage{graphicx, ucs}

\usepackage{subcaption}
\usepackage{float}
\usepackage{tikz}

\usepackage{enumitem}
\usepackage{hyperref}
\hypersetup{
  colorlinks = true,
  urlcolor = {blueGrotto},
  linkcolor = {royalBlue},
  citecolor = {limeGreen}
}
\usepackage{array}

\usepackage[outline]{contour}
\usepackage{xcolor}

\usepackage{nicefrac}       
\usepackage{multirow}
\usepackage{booktabs}
\usepackage{xcolor}

\newcommand{\N}{\ensuremath{\mathbb{N}}}

\newcommand{\R}{\ensuremath{\mathbb{R}}}

\newcommand{\calA}{\ensuremath{\mathcal{A}}}

\newcommand{\calD}{\ensuremath{\mathcal{D}}}

\newcommand{\calL}{\ensuremath{\mathcal{L}}}

\newcommand{\calU}{\ensuremath{\mathcal{U}}}

\renewcommand{\vec}[1]{\boldsymbol{#1}}
\newcommand{\matr}[1]{\boldsymbol{#1}}

\newcommand{\matA}{\ensuremath{\boldsymbol{A}}}

\newcommand{\matJ}{\ensuremath{\boldsymbol{J}}}

\newcommand{\matP}{\ensuremath{\boldsymbol{P}}}

\newcommand{\veca}{\ensuremath{\boldsymbol{a}}}
\newcommand{\vecb}{\ensuremath{\boldsymbol{b}}}

\newcommand{\vecf}{\ensuremath{\boldsymbol{f}}}
\newcommand{\vecg}{\ensuremath{\boldsymbol{g}}}
\newcommand{\vech}{\ensuremath{\boldsymbol{h}}}

\newcommand{\vecq}{\ensuremath{\boldsymbol{q}}}
\newcommand{\vecr}{\ensuremath{\boldsymbol{r}}}

\newcommand{\vecu}{\ensuremath{\boldsymbol{u}}}
\newcommand{\vecv}{\ensuremath{\boldsymbol{v}}}
\newcommand{\vecw}{\ensuremath{\boldsymbol{w}}}
\newcommand{\vecx}{\ensuremath{\boldsymbol{x}}}
\newcommand{\vecy}{\ensuremath{\boldsymbol{y}}}
\newcommand{\vecz}{\ensuremath{\boldsymbol{z}}}

\theoremstyle{plain}            
\newtheorem{theorem}{Theorem}[section]
\newtheorem{inftheorem}[theorem]{Informal Theorem}
\newtheorem{lemma}[theorem]{Lemma}
\newtheorem{corollary}[theorem]{Corollary}
\newtheorem{proposition}[theorem]{Proposition}
\newtheorem{claim}[theorem]{Claim}

\theoremstyle{definition}       
\newtheorem{definition}[theorem]{Definition}

\theoremstyle{remark}           

\numberwithin{equation}{section}

\contourlength{0.1pt}
\contournumber{10}

\newif\ifnotes\notestrue

\ifnotes
\usepackage{color}
\definecolor{mygrey}{gray}{0.50}
\newcommand{\notename}[2]{{\textcolor{red}{\footnotesize{\bf (#1:} {#2}{\bf
) }}}}

\else

\newcommand{\notename}[2]{{}}

\fi

\DeclareMathOperator*{\Exp}{{\mathbb{E}}}
\DeclareMathOperator*{\Prob}{{\mathbb{P}}}

\DeclareMathOperator*{\KL}{\mathrm{D}_{\mathrm{KL}}}

\mathchardef\mdash="2D 

\newcommand{\eps}{\varepsilon}

\renewcommand{\epsilon}{\varepsilon}

\def\compactify{\itemsep=0pt \topsep=0pt \partopsep=0pt \parsep=0pt}
\let\latexusecounter=\usecounter

\newenvironment{Enumerate}
  {\def\usecounter{\compactify\latexusecounter}
   \begin{enumerate}}
  {\end{enumerate}\let\usecounter=\latexusecounter}

{\begin{itemize}%
\setlength{\itemsep}{0pt}%
\setlength{\topsep}{0pt}%
\setlength{\partopsep}{0 in}%
\setlength{\parskip}{0 pt}}%
{\end{itemize}}

\DeclareMathOperator*{\argmax}{argmax}

\def\abs#1{\left|#1\right|}
\def\p#1{\left(#1\right)}
\def\b#1{\left[#1\right]}
\def\set#1{\left\{#1\right\}}

\newcommand{\paragr}[1]{\noindent \textbf{#1}}

\def\norm#1{\left\|#1\right\|}

\def\OPT{\mathrm{OPT}}

\newcommand{\reals}{\mathbb{R}}
\newcommand{\realsp}{\mathbb{R}_{+}}

\newcommand{\nats}{\mathbb{N}}

\definecolor{secinhead}{RGB}{249,196,95}
\definecolor{niceRed}{RGB}{190,38,38}
\definecolor{blueGrotto}{HTML}{059DC0}
\definecolor{royalBlue}{HTML}{057DCD}
\definecolor{navyBlueP}{HTML}{0B579C}
\definecolor{limeGreen}{HTML}{81B622}

\def\ren{R\'enyi}
\def\plsm{\textsc{PLSoftMax\xspace}}
\def\logplsm{\textsc{LogPLSoftMax\xspace}}

\newcommand{\Ground}{\mathcal{M}}
\newcommand{\Distr}{\mathcal{F}}
\newcommand{\Domain}{\mathcal{D}}
\newcommand{\Image}{\Omega}

\newcommand{\Revs}{\textsc{Rev}}
\newcommand{\Rev}[1]{\textsc{Rev}(#1)}
\newcommand{\OptM}[1]{\textsc{OptRev}(#1)}
\newcommand{\Revf}{\textsc{Rev}}

\renewcommand{\vec}[1]{\boldsymbol{#1}}

\newcommand{\diver}{\mathrm{D}}
\newcommand{\renyi}[3]{\mathrm{D}_{#1}\left( #2 \| #3 \right)}

\newcommand{\sign}{\text{sign}}

\newcommand{\Expon}{\textsc{Exp}}
\newcommand{\Rec}{\plsm}

\newcommand{\Pow}{\textsc{Pow}}

\renewcommand{\div}{\text{ div }}

\begin{document}

\title{Optimal Approximation - Smoothness Tradeoffs for Soft-Max Functions}
\author{
  \textbf{Alessandro Epasto} \\
  \small Google Research \\
  \url{aepasto@google.com}
  \and
  \textbf{Mohammad Mahdian} \\
  \small Google Research \\
  \url{mahdian@google.com}
  \and
  \textbf{Vahab Mirrokni} \\
  \small Google Research \\
  \url{mirrokni@google.com}
  \and
  \textbf{Manolis Zampetakis} \\
  \small Massachusetts Institute of Technology \\
  \url{mzampet@mit.edu}
}
\maketitle

\begin{abstract}
    A soft-max function has two main efficiency measures: (1) \textit{approximation} - which corresponds
  to how well it approximates the maximum function, (2) \textit{smoothness} - which
  shows how sensitive it is to changes of its input.
  Our goal is to identify the optimal approximation-smoothness tradeoffs for different
  measures of approximation and smoothness. This leads to novel soft-max functions, each of which is
  optimal for a different application.
  The most commonly used
  soft-max function, called exponential mechanism, has optimal tradeoff between approximation measured in
  terms of \textit{expected additive approximation} and smoothness measured with respect to 
  \textit{\ren\ Divergence}. We introduce a soft-max function, called
  \textit{piecewise linear soft-max}, with optimal tradeoff between approximation, measured in terms of 
  \textit{worst-case} additive approximation and smoothness, measured with respect to
  \textit{$\ell_q$-norm}. The worst-case approximation guarantee of the piecewise linear mechanism enforces
  \textit{sparsity} in the output of our soft-max function, a property that is known to 
  be important in Machine Learning applications \cite{MartinsA16, LahaCAJSR18} and is
  not satisfied by the exponential mechanism. Moreover, the \textit{$\ell_q$}-smoothness is suitable
  for applications in Mechanism Design and Game Theory where the piecewise linear mechanism outperforms the
  exponential mechanism. Finally, we investigate another soft-max function, called \textit{power
  mechanism}, with optimal tradeoff between expected \textit{multiplicative} approximation and
  smoothness with respect to the R\'enyi Divergence, which provides improved theoretical and practical
  results in differentially private submodular optimization.
\end{abstract}

\section{Introduction} \label{sec:intro}
  
  A soft-max function is a mechanism for choosing one out of a number of options, given the
value of each option. Such functions have applications in many areas of computer science and
machine learning, such as deep learning 
(as the final layer of a neural network classifier) 
\cite{NIPS1989_195,Bridle,Goodfellow-et-al-2016}, reinforcement learning 
(as a method for selecting an action)~\cite{RLBook}, learning from mixtures of
experts~\cite{JordanJacobs}, differential privacy~\cite{dworkr14,McSherryT07}, and mechanism
design~\cite{McSherryT07,ExpMechDesign}. The common requisite in these applications is for the
soft-max function to pick an option with close-to-maximum value, while behaving smoothly as the
input changes.

  The soft-max function that has come to dominate these applications is the 
{\em exponential function}. Given $d$ options with values $x_1, x_2, \ldots, x_d$, the
exponential mechanism picks $i$ with probability equal to the quantity
$\exp(\lambda x_i)/(\sum_{j=1}^d\exp(\lambda x_j))$ for a parameter $\lambda > 0$. This
function has a long history: It has been proposed as a model in decision theory in 1959 by
Luce~\cite{Luce1959}, and has its roots in the Boltzman (also known as Gibbs) distribution in
statistical mechanics~\cite{Boltzman,gibbs1902}. There are, however, many other ways to
smoothly pick an approximately maximal element from a list of values. This raises the
question: is there a way to quantify the desirable properties of soft-max functions, and are
there other soft-max functions that perform well under such criteria? If there are such
functions, perhaps they can be added to our repertoire of soft-max functions and might prove
suitable in some applications.

  These questions are the subject of this paper. We explore the tradeoff between the
approximation guarantee of a soft-max function and its smoothness. A soft-max function is
\textit{$\delta$-approximate} if the expected value of the option it picks is at least the
maximum value minus $\delta$. Stronger yet, a function is 
\textit{$\delta$-approximate in the worst case} if it never picks an option of value less than
the maximum minus $\delta$. We capture the requirement of {\em smoothness} using the notion of
Lipschitz continuity. A function is Lipschitz continuous if by changing its input by some
amount $x$, its output changes by at most a multiple of $x$. This multiplier, known as the
Lipschitz constant, is then a measure of smoothness. This notion requires
a way to measure distances in the domain (the input space) and the range (the output space) of
the function. 

  We will show that if the $p$-norm and the R\'enyi divergence are used to measure distances in
the domain and the range, respectively, then the exponential mechanism achieves the lowest
possible (to within a constant factor) Lipschitz constant among all $\delta$-approximate
soft-max functions. This Lipschitz constant is $O(\log(d)/\delta)$. The exponential function
picks each option with a non-zero probability, and therefore cannot guarantee worst-case approximation. In fact, we show that for these distance measures, there is no soft-max
function with bounded Lipschitz constant that can guarantee worst-case approximation.

  On the other hand, if we use $p$-norms to measure changes in both the input and the output,
new possibilities open up. We construct a soft-max function 
(called \plsm, for piecewise linear soft-max) that achieves a Lipschitz constant of
$O(1/\delta)$ and is also $\delta$-approximate {\em in the worst case}. This is an important
property, as it guarantees that the output of the soft-max function is always as sparse as
possible. Furthermore, we prove that even only requiring $\delta$-approximation in
expectation, no soft-max function can achieve a Lipschitz constant of $o(1/\delta)$ for these
distance measures.

  We also study several other properties we might want to require of a soft-max function. Most
notably, what happens if instead of requiring an additive approximation guarantee, we require
a multiplicative one? A simple way to construct a soft-max function satisfying this requirement
is to apply soft-max functions with additive approximation (e.g., exponential or \plsm) on the
logarithm of the values. The resulting mechanisms (the power mechanism, and \logplsm) are
Lipschitz continuous, but with respect to a domain distance measure called log-Euclidean.
Moreover, we show that with the standard $p$-norm distance 
as the domain distance measure, no soft-max function with bounded Lipschitz constant and
multiplicative approximation guarantee exists.

  Finally, we explore several applications of the new soft-max functions introduced in this
paper. First, we show how the power mechanism can be used to improve existing results 
(using the exponential mechanism) on differentially private submodular maximization. Second,
we use \plsm\ to design improved incentive compatible mechanisms with worst-case guarantees.
Finally, we discuss how \plsm\ can be used as the final layer of deep neural networks in
multiclass classification. 

\subsection{Related Work} \label{sec:intro:related}
  
  A lot of work has been done in designing soft-max function that fit better to specific
applications. In Deep Learning applications, the exponential mechanism does not allow to take
advantage of the sparsity of the categorical targets during the training. Several methods
have been proposed to take use of this sparsity. Hierarchical soft-max uses a heuristically
defined hierarchical tree to define a soft-max function with only a few outputs
\cite{morin2005hierarchical, mikolov2013distributed}. Another direction is the use of a
spherically symmetric soft-max function together with a spherical class of loss functions that
can be used to perform back-propagation step much more efficiently 
\cite{vincent2015efficient, de2015exploration}. Finally there has been a line of work that
targets the design of soft-max functions whose output favors sparse distributions
\cite{MartinsA16, LahaCAJSR18}.
\section{Definitions and Preliminaries} \label{sec:model}
  The $(d-1)$-dimensional {\em unit simplex} (also known as the {\em probability simplex}) is
the set of all the $d$-dimensional vectors $(x_1,\ldots,x_d)\in\R^d$ satisfying $x_i\ge 0$ for all $i$ and
$\sum_{i=1}^d x_i = 1$. In other words, each point in the $(d-1)$-dimensional unit simplex,
which we denote by $\Delta_{d-1}$, corresponds to a probability distribution over $d$ possible
outcomes $1,\ldots,d$.

\paragraph{Soft-max. }
  A $d$-dimensional {\em soft-max function} (sometimes called a soft-max mechanism) is a
function $\vecf:\mathbb{R}^d \to \Delta_{d-1}$. Intuitively, this corresponds to a
randomized mechanism for choosing one outcome out of $d$ possible outcomes. For any 
$\vecx \in\R^d$, the value $x_i$ denotes the value of the outcome $i$, and $f_i(\vecx)$ is the
probability that $f$ chooses this outcome. In parts of this paper, specifically when we
discuss multiplicative approximations, we restrict the outcome values $x_i$ to be positive,
i.e., we consider soft-max functions from $\R_+^d$ to $\Delta_{d-1}$.
\smallskip

\paragr{Lipschitz continuity.}
  The Lipschitz property is defined in terms of a distance measure $d_1$ over $\R^d$ 
(the domain) and a distance measure $d_2$ over $\Delta_{d-1}$ (the range). A distance measure
over a set is a function that assigns a non-negative {\em distance} to every ordered pair of
points in that set. We do not require symmetry or the triangle inequality. We say that a
soft-max function $f$ is $(d_1,d_2)$-Lipschitz continuous if there is a constant $\beta > 0$
such that for every two points $\vecx, \vecy \in \R^d$, the following holds
\begin{equation} \label{eq:intro:Lipschitzness}
    d_2(\vecf(\vecx), \vecf(\vecy)) \le \beta \cdot d_1(\vecx, \vecy).
\end{equation}
The smallest $\beta$ for which \eqref{eq:intro:Lipschitzness} holds is the Lipschitz
constant of $\vecf$ (with respect to $d_1$ and $d_2$).
\smallskip

\paragr{$\ell_p$ distance and \ren\ divergence.}
  Two measures of distance that are used in this paper are the $p$-norm distance and the \ren\
divergence. For $p \ge 1$, the $p$-norm distance (also called the $\ell_p$ distance) between
two points $\vecx, \vecy \in \R^d$ is denoted by $\norm{\vecx - \vecy}_p$, and is defined as 
$\norm{\vecx - \vecy}_p = \left(\sum_{i = 1}^d \abs{x_i - y_i}^p\right)^{1/p}$. For any
$\alpha>1$ and points $\vecx, \vecy \in \Delta_{d - 1}$, the R\'enyi divergence of order
$\alpha$ between $\vecx$ and $\vecy$ is denoted by $\diver_\alpha(\vecx||\vecy)$ and is
defined as 
$\diver_\alpha(x||y) = \frac{1}{\alpha - 1} \log\p{ \sum_{i = 1}^d \frac{x_i^{\alpha}}{y_i^{ \alpha - 1}}}$.
This expression is undefined at $\alpha=1$, but the limit as $\alpha \to 1$ can be written as
$D_1(\vecx||\vecy)=\sum_{i=1}^d x_i\log\frac{x_i}{y_i}$ and is known as the Kullback-Leibler
(KL) divergence. Similarly, the \ren\ divergence of order $\infty$ can be defined as the limit
as $\alpha \to \infty$, which is $D_\infty(\vecx||\vecy) = \log\max_{i}\frac{x_i}{y_i}$.
\smallskip

\paragr{Approximation.}
  For any $\delta \ge 0$, a soft-max function $\vecf : \mathbb{R}^d \mapsto\Delta_{d-1}$ is 
{\em $\delta$-approximate} if
\begin{equation}
  \forall \vecx \in \R^d~:\qquad \langle \vecx, \vecf(\vecx) \rangle \ge \max_i\{x_i\} - \delta.
\end{equation}
Note that the inner product $\langle \vecx, \vecf(\vecx) \rangle$ is the expected value of the
outcome picked by $\vecf$. The function $\vecf$ is 
{\em $\delta$-approximate in the worst case} if
\begin{equation}
  \forall \vecx \in \R^d, \forall i \in [d] ~:\qquad f_i(\vecx) > 0\Rightarrow x_i \ge \max_i\{x_i\} - \delta.
\end{equation}
 
\section{The Exponential Mechanism} 
\label{sec:exponential}
  The exponential soft-max function, parameterized by a parameter $\lambda$ and denoted by
$\Expon^\lambda$, is defined as follows: for $\vecx \in \R^d$, $\Expon^\lambda(\vecx)$ is a
vector whose $i$'th coordinate is $\exp(\lambda x_i)/\sum_{j = 1}^d\exp(\lambda x_j)$.

  This mechanism was proposed and analyzed by McSherry and Talwar~\cite{McSherryT07} for its
application in differential privacy and mechanism design. It is not hard to see that the
differential privacy property they prove corresponds to $(\ell_p, D_\infty)$-Lipschitz
continuity, and therefore their analysis, cast in our terminology, implies the following:

\begin{theorem}[\cite{McSherryT07}] \label{thm:exp}
    For any $\delta > 0$ and $p, \alpha\ge1$, the soft-max function $\Expon^\lambda$ with
  $\lambda=\log(d)/\delta$ satisfies the following: (1) it is $\delta$-approximate, and (2) it
  is $(\ell_p, D_{\alpha})$-Lipschitz continuous with a Lipschitz less than $2 \lambda$.
\end{theorem}

  This leaves the following question: is there any other soft-max function that achieves a
better Lipschitz constant? The following theorem gives a negative answer.

\begin{theorem} \label{thm:expLB}
    Let $p, \alpha \ge 1$, $\delta > 0$, $d \ge 4$ and $\vecf : \R^d \to \Delta_{d-1}$ be a
  $\delta$-approximate soft-max function satisfying 
  $\renyi{\alpha}{\vec{f}(\vec{x})}{\vec{f}(\vec{y})} \le c \norm{\vec{x} - \vec{y}}_p$ for
  all $\vec{x}, \vec{y} \in \R^d$. Then it holds $c > \frac{\log d - 2}{4 \delta}$.
\end{theorem}

  Also, since the exponential mechanism assigns a non-zero probability to any outcome, it is
of course not $\delta$-approximate in the worst case. The following theorem shows that this is
an unavoidable property of any $(\ell_p, D_{\alpha})$-Lipschitz continuous functions.

\begin{theorem} \label{thm:expNoWorstCase}
    For any $p, \alpha \ge 1$, $\delta>0$, there is no soft-max function that is 
  $(\ell_p, D_{\alpha})$-Lipschitz continuous and $\delta$-approximate in the worst case.
\end{theorem}

The proofs of the above theorems are presented in~\refappendix{app:expLB}.
\section{\plsm: A Soft-Max Function with Worst Case Guarantee} \label{sec:plsm}
As we saw in the last section, the exponential mechanism is a $(\ell_p,D_\infty)$-Lipschitz function with the best possible Lipschitz constant among all $\delta$-approximate functions. Furthermore, a worst case approximation guarantee is not possible for such Lipschitz functions. In this section, we focus on $(\ell_p,\ell_q)$-Lipschitz functions which are the soft-max functions that are used in mechanism design and in machine learning setting. These functions exhibit a different picture: the exponential function is no longer the best function in this family. Instead, we construct a soft-max function that achieves the best (up to a constant factor) Lipschitz constant and at the same time provides a worst-case guarantee. This is the most technical result of the paper.

\subsection{Construction of \plsm}
While the analysis of the properties of \plsm\ and understanding the intuition behind its construction might be technically challenging, its actual description is rather concise and simple. In this Section we give a complete description of this soft-max function, and state our main result. Due to lack of space, the proofs are left to~\refappendix{app:plsm}.

\plsm\ is a piecewise linear function, where each linear piece is defined using a carefully
designed matrix. More precisely, for a given $\vecx\in\R^d$, consider a permutation $\pi$ of $\{1,\ldots,d\}$ that {\em sorts} $x$, i.e., $x_{\pi(1)} \ge x_{\pi(2)} \ge \cdots \ge x_{\pi(d)}$, and let $\matP_{\pi}$ be the permutation matrix of $\pi$, i.e., the matrix with $1$'s at entries $(i,\pi(i))$ and zeros everywhere else. In other words, $\matP_{\pi}$ is the matrix that, once multiplied by $x$, sorts it. Each ``piece'' of our piecewise linear function corresponds to all $x\in\R^d$ that have the same sorting permutation $\pi$. The function, on this piece, is defined by multiplying $\vecx$ by $\matP_{\pi}$ (thereby sorting it), then applying a linear function defined through a carefully designed family of matrices $\matr{SM}_{(k, d)}$, and then applying the inverse matrix $\matP_{\pi}^{-1}$ to move values back to their original index. The matrices $\matr{SM}_{(k, d)}$ at the heart of this construction are defined below.

\begin{definition}[\textsc{Soft-Max Matrix}]
The soft max matrix $\matr{SM}_{(k, d)} = (m_{i j}) \in \reals^{d \times d}$ is defined as
$m_{1 1} = (k - 1)/k$, $m_{i i} = 1/i$ for all $i \in [2, k]$, $m_{i 1} = - 1/k$ for all 
$i \in [2, k]$, $m_{i j} = -1/(j (j - 1))$ for all $i, j \in [2, k]$ with $j < i$, and 
$m_{i j} = 0$ otherwise
(See Appendix~\ref{app:soft-maxmat} for a better illustration of this matrix). Also, the vector
$\vecu^{(k)} \in \R^d$ is defined as $u^{(k)}_i = 1/k$ if $i \le k$ and $u^{(k)}_i = 0$ otherwise.
\end{definition}
  
We consider partitions where each piece contains all vectors with the same ordering of the
coordinates. Namely, for a permutation $\pi \in S_d$ we define $R_{\pi}$ to be the set of
vectors $\vecx \in \R^d$ such that $x_{\pi(1)} \ge x_{\pi(2)} \ge \cdots \ge x_{\pi(d)}$. Also, let
$\matP_{\pi}$ be the permutation matrix of $\pi \in S_d$. 

\begin{definition}(\plsm) \label{def:soft-maxFunctionDef}
    Let $\delta > 0$, and consider a vector $\vecx \in \R^d$ with a sorting permutation $\pi$ and the
  corresponding permutation matrix $\matP_\pi$. Define $k_{\vecx}$ as the maximum $k \in [d]$ such that
  $x_{\pi(1)} - x_{\pi(k)}\le \delta$. The soft-max function~$\plsm^\delta$\ on $x$ is defined as 
  follows.
  \begin{equation} \label{eq:soft-maxFunctionTV}
    \plsm^{\delta}(\vecx) = \frac{1}{\delta} \cdot \matP_{\pi}^{-1} \cdot \matr{SM}_{(k_{\vecx} ,d)} \cdot \matP_{\pi} \cdot \vecx + \matP_{\pi}^{-1} \cdot u^{(k_{\vecx})}.
  \end{equation}
\end{definition}

As defined, it is not even clear that $\plsm^\delta$ is a valid soft-max function, i.e., that $\plsm^\delta(x)\in\Delta_{d-1}$. This, as well as the following result, is proved in Appendix~\ref{app:plsm}.

\medskip
\begin{theorem} \label{thm:mainsoft-maxFunctionTV}
Let $\delta > 0$, $\plsm^{\delta}$ be the function defined in
\eqref{eq:soft-maxFunctionTV} and let $\vec{x} \in \R^d$, then
\begin{Enumerate}
\item $\plsm^{\delta}$ is $\delta$-approximate in the worst case.
\item For any $p, q \ge 1$, $\plsm^{\delta}$ is $(\ell_p,\ell_q)$-Lipschitz continuous with a Lipschitz constant that is at most 
\[\frac{2}{\delta} \min\{p + 1, \frac{q}{q - 1}, \log d\}.\]
\end{Enumerate}
\end{theorem}

The proof of Theorem \ref{thm:mainsoft-maxFunctionTV} is based on bounding the
$(p, q)$-subordinate norm of the matrices $\matr{SM}_{(k, d)}$. This is a challenging task since even computing the $(p, q)$-subordinate norm is NP-hard~\cite{Rohn2000, HendrickxO10}. To circumvent this, we
generalize a theorem of~\cite{DrakakisP09} for subordinate norms, which might be of independent interest.

\subsection{Lower Bounds and Comparison with the Exponential Function} 
\label{sec:plsm_LB}

Theorem~\ref{thm:mainsoft-maxFunctionTV} shows that the $(\ell_p, \ell_q)$-Lipschitz constant of \plsm\ is at most $O(1/\delta)$ when $p$ is bounded or when $q$ is bounded away from $1$, but becomes $O(\log(d)/\delta)$ when $(p,q)$ gets close to $(\infty,1)$. It is easy to see that no soft-max function can achieve a Lipschitz constant better than $O(1/\delta)$. The following theorem shows that even for $(p,q)=(\infty,1)$, no soft-max function can beat the bound proved in Theorem~\ref{thm:mainsoft-maxFunctionTV} for \plsm. The proofs of this theorem and the other theorems in this Section are deferred to Appendix~\ref{app:plsm_LB}.

\begin{theorem} \label{thm:totalVariationLowerBound}
Let $c,\delta>0$, and assume $\vecf: \R^d \to \Delta_{d-1}$ is a soft-max function that is $\delta$-approximate and $(\ell_\infty,\ell_1)$-Lipschitz continuous with a Lipschitz constant of at most $c$. Then, $c \ge \log_2 d/(4 \delta)$.
\end{theorem}

It is not hard to prove that for every $\vecx, \vecy$, $\norm{\vecx - \vecy}_1 \le \renyi{\infty}{\vecx}{\vecy}$. Therefore, since the exponential soft-max function $\Expon^\lambda$ for $\lambda=\log(d)/\delta$ is $(\ell_p,D_\infty)$-Lipschitz continuous with a Lipschitz constant of at most $2\lambda$ (Theorem~\ref{thm:exp}), it must also be $(\ell_p,\ell_1)$-Lipschitz with the same constant. The following theorem shows that this Lipschitz constant is at least $\frac{\lambda}2$. 

\begin{theorem} \label{thm:exponentialLowerBound}
The $(\ell_p,\ell_1)$-Lipschitz constant of the soft-max function $\Expon^\lambda$ is at least $\frac{\lambda}2$. Therefore, the $(\ell_p,\ell_1)$-Lipschitz constant of a $\delta$-approximate exponential soft-max function is at least $\frac{\log d}{4 \delta}.$
\end{theorem}

The combination of the above result and Theorem~\ref{thm:mainsoft-maxFunctionTV} shows that in terms of the $(\ell_p,\ell_1)$-Lipschitz constant, there is a gap of $\Theta(\log d)$ between the exponential function and \plsm.
\section{Other variants and desirable properties} 
\label{sec:otherproperties}
In the previous sections, we studied the tradeoff between Lipschitz continuity of soft-max functions and their approximation quality, as quantified by the maximum additive gap between the (expected) value of the outcome picked and the maximum value. In this section, we look into variants of our definitions and other desirable properties that we might need to require from the soft-max function. Most importantly, is it possible to require a multiplicative notion of approximation?

\subsection{Multiplicative approximation}
\label{sec:multiplicative}
For any $\delta \ge 0$, we call a soft-max function $\vecf : \R_+^d\mapsto\Delta_{d-1}$ is {\em $\delta$-multiplicative-approximate} if for every $\vecx \in \R_+^d$, we have $\langle \vecx, \vecf(\vecx) \rangle \ge (1 - \delta) \max_i\{x_i\}$. Similarly, we can define the notion of $\delta$-multiplicative-approximate in the worst case.\footnote{Note that throughout this section, we restrict the domain of the soft-max function to only positive values.} Such multiplicative notions of approximation are practically useful in settings where the scale of the input is unknown. 

First, here is a simple observation: to get a soft-max function with a multiplicative approximation guarantee, it is enough to start with one with an additive guarantee and apply it to the logarithm of the input values. The resulting function will be Lipschitz continuous, but with respect to a different distance measure as defined below.

\begin{definition}
    For any $\vecx \in \R_+^d$, let $\log(\vecx) := (\log(x_1),\ldots,\log(x_d))$. 
  For $p\ge 1$, the \textit{$p$-log-Euclidean} distance between two points $\vecx, \vecy \in \R_+^d$ is denoted by $\text{Log-}\ell_p(\vecx, \vecy)$ and is defined as $\text{Log-}\ell_p(\vecx, \vecy):=\ell_p(\log(\vecx), \log(\vecy))$. Note that $\text{Log-}\ell_p$ is a metric.
\end{definition}

We can now state the above observation as follows:

\begin{proposition}
  Let $\vecf:\R^d\to\Delta_{d-1}$ be a soft-max function that is $\delta$-approximate and $(\ell_p, \chi)$-Lipschitz for a distance measure $\chi$. Then the function $\text{Log}\vecf:\mathbb{R_+}^d\mapsto\Delta_{d-1}$ defined by $\text{Log}\vecf(x):=\vecf(\log(x))$ is a $\delta$-multiplicative-approximate soft-max function that is $(\text{Log-}\ell_p, \chi)$-Lipschitz with the same Lipschitz constant as $\vecf$.
\end{proposition}

Applying this proposition to \plsm, we obtain a soft-max function called \logplsm\ that is $\delta$-multiplicative-approximate in the worst case and $(\text{Log-}\ell_p, \ell_q)$-Lipschitz. 
More notably, applying this proposition to the exponential function, we obtain a soft-max function that we call the {\em power mechanism}, with a very simple and natural description: The Power Mechanism $\Pow^{\lambda}$ with parameter $\lambda$, 
applied to the input vector $x\in \realsp^d$ is defined as 
$\Pow^{\lambda}_i(\vecx) = x_i^{\lambda}/\sum_{j = 1}^d x_j^{\lambda}$. We will see in Section~\ref{sec:applications} how this mechanism can be used to improve existing results in a differentially private optimization problem.

A question that remains is whether, to obtain a multiplicative approximation, it is necessary to switch the domain distance measure to $\text{Log-}\ell_p$. In other words, are there $\delta$-multiplicative-approximate soft-max functions that are Lipschitz with respect to the domain metric $\ell_p$? The following theorem, whose proof is deferred to the appendix, provides a negative answer.

\begin{theorem} 
    For $\delta > 0$, let $\vecf : \R^d \to \Delta_{d-1}$ be a $\delta$-multiplicative-approximate
  soft-max function. Then there is no $p, q$ such that $\vecf$ is $(\ell_p,\ell_q)$-Lipschitz with a bounded Lipschitz constant. Similarly, there is no $p, \alpha$ such that $f$ is $(\ell_p,D_\alpha)$-Lipschitz with a bounded Lipschitz constant. 
\end{theorem}

\subsection{Scale and Translation Invariance}
Related to the notion of multiplicative approximation, one might wonder if there are soft-max functions that are {\em scale invariant}, i.e., guarantee that for every $\vecx \in \R^d$ and $c \in \R$, $\vecf(c \vecx) = \vecf(\vecx)$? Similarly, one may require {\em translation invariance}, i.e., that for every $\vecx \in \R^d$ and $c \in \R$, $\vecf(\vecx + c \cdot {\bf 1}) = \vecf(\vecx)$. It is easy to see that indeed the mechanisms \Expon\ and \Pow\ are  translation and scale invariant, respectively. It is less obvious, but still not difficult, to show that similarly, the mechanisms \plsm\ and \logplsm\ are translation and scale invariant, respectively.

In fact, it turns out that translation and scale invariance go hand-in-hand with the notion of approximation: no scale-invariant function can guarantee additive approximation, and no translation-invariant function can guarantee multiplicative approximation. 
\section{Applications} 
\label{sec:applications}
%
We present three applications of the soft-max functions introduced in this paper. In Section~\ref{sec:mech_design}, we show how to use \plsm\ to design approximately incentive compatible mechanisms. In Section~\ref{sec:dp_submodular} we use \Pow\ to improve a result on differentially private submodular maximization. Finally, in Section~\ref{sec:sparse_classification}, we discuss potential applications of \plsm\ in neural network classifiers.

\subsection{Approximately Incentive Compatibile Mechanisms via \plsm}
\label{sec:mech_design}
Let us start with an abstract definition of incentive compatibility in mechanism design.
Consider a setting with $n$ self-interested agents, indexed $1,\ldots, n$. A mechanism is a (randomized) algorithm $A$ that must pick one of the possible outcomes in a set $\Omega$. For simplicity, let us assume that $\Omega$ is finite and $|\Omega|=d$. Each agent $i$ has a utility function $u_i\in \R_+^\Omega$ that specifies the value that $i$ places on each of the possible outcomes. Let $\calU \subseteq \R_+^{\Omega}$ denote the space of all possible utility functions. The input of the algorithm $A$ is the reported utility of all the agents, i.e., $A$ takes a $u\in \calU^n$ as input, and probabilistically picks an outcome $A(u)$ in $\Omega$. We say that $A$ is \textit{$\eps$-incentive compatible} with respect to $\calU$ if for every $u \in \calU^n$, $u' \in \calU$ and every agent $i$, the following inequality holds
$\Exp_{z \sim A(\vecu)}\b{u_i(z)} \ge \Exp_{z \sim A(u', \vecu_{-i})}\b{u_i(z)} - \eps$.

Typically, in mechanism design, the challenge is to design a mechanism $A$ that is incentive compatible and at the same time (approximately) optimizes a given objective function $w$ that depends on the utility of the agents $u\in\calU^n$ as well as the selected outcome in $\Omega$. 
At a high-level, a soft-max function can be used to design an incentive compatible mechanism as follows: Assume $f:\R^d\to\Delta_{d-1}$ is $(\chi,\ell_1)$-Lipschitz with respect to some domain distance measure $\chi$. The mechanism $A_f$ is defined as follows: it computes the value of all outcomes in $\Omega$ at the reported utilities $u\in \calU^n$, and uses $f$ to pick an outcome with respect to these values. 

A central concept is the sensitivity of the function $w$ with respect to $\chi$. The $\chi$-sensitivity $S_{\chi}(w)$ of $w$ is defined as
$S_{\chi}(w) = \max\{\chi \left(\vecw(\vecv), \vecw(\vecv_{-i}, v'_i) \right)\}$, where 
$\vecw(\vecv) = (w(\vecv, 1), \dots, w(\vecv, d))$ and the maximum is taken over all possible 
$i, \vecv, v_i'$. If the soft-max function $f$ has low $(\chi,\ell_1)$-Lipschitz constant, and the objective $w$ has low sensitivity with respect to $\chi$, we can use the following theorem to obtain an $\eps$-incentive compatible mechanism.

\begin{theorem} \label{lem:LipschitzFromFtoA}
Assume a mechanism design setting where utilities of the agents are bounded from above by 1, i.e., 
$\calU \subseteq [0, 1]^{\Omega}$.
Let $f : \R^d \to \Delta_{d}$ be a soft-max function with $(\chi,\ell_1)$-Lipschitz constant at most $L$, and $w : \calU^n\times\Omega \to \R$ be an objective function. The 
algorithm $A_f$ is $(L/S_{\chi}(w))$-incentive compatible with respect to $\calU$.
\end{theorem}

The connection between soft-max and mechanism design established in the above theorem is not new. McSherry and Talwar~\cite{McSherryT07} originally pointed out this connection and used it to design incentive compatible mechanisms. However, they stated this connection in terms of differential privacy (closely related to the $(\chi,D_\infty)$-Lipschitz property). The main difference between the above theorem and the one by McSherry and Talwar is that we only require $(\chi,\ell_1)$-Lipschitz continuity, which is closer to what the application demands. This can be combined with the soft-max function \plsm\ analyzed in Theorem \ref{thm:mainsoft-maxFunctionTV} to obtain results that were not achievable using the exponential mechanism. We present two applications of this here. See~\refappendix{sec:singleItem} for details and proofs.

\paragraph{Worst-Case Guarantees for Mechanism Design.} 
If we replace the exponential mechanism with \plsm\ in many applications of
Differential Privacy in Mechanism Design, we get approximate incentive compatible algorithms with
\textit{worst-case} approximation guarantees as opposed to the expected approximation or the
high-probability guarantees that are currently known. Consider for example the digital goods
auction problem from \cite{McSherryT07}, where $n$ bidders have a private utility for a good at
hand for which the auctioneer has an unlimited supply and let $\mathrm{R}_{\OPT}$ be the optimal
revenue that the auctioneer can extract for a given set of bids. We can then prove the following.

\begin{inftheorem} \label{ithm:digitalGoodsWorstCase}
    There is an $\eps$-incentive compatible mechanism for the digital goods auction problem where 
  the revenue of the auctioneer is at least 
  $\mathrm{R}_{\OPT} - O(\mathrm{R}_{\OPT} \cdot n)/\eps)$ \textbf{in the worst-case}.
\end{inftheorem}

\paragraph{Better Sensitivity implies better Utility.} If the revenue objective
function $w$ has bounded $S_q(w)$ sensitivity for some $q < \log(d)$, then using \plsm\ we get a significantly better revenue-incentive compatibility tradeoff compared to using
the exponential mechanism. This is clear form Lemma \ref{lem:LipschitzFromFtoA}, and Theorems 
\ref{thm:mainsoft-maxFunctionTV} and \ref{thm:exp}. See~\refappendix{sec:singleItem} for details.

\subsection{Differentially Private Submodular Maximization via the Power Mechanism}
\label{sec:dp_submodular}
We now show that $D_{\infty}$ smoothness can be used to design differentially private algorithms.

\paragr{Differential Privacy.} A randomized algorithm $A$ satisfies
$\eps$-\textit{differential privacy} if
$\Prob(A(\vecv) \in S) \le \exp(\eps) \cdot \Prob(A(v'_i, \vec{v}_{-i}) \in S)$
for all $i \in [n]$, $\vecv \in \Domain^n$, $v'_i \in \Domain$ and all sets 
$S \subseteq \Image$, where $\Image$ is the set of possible outputs of $A$.
\smallskip

 
 For some distance metric $\chi$ of $\R_+^d$, the soft-max function $\vecf$ satisfies $\renyi{\infty}{\vec{f}(\vec{x})}{\vec{f}(\vec{y})} \le   L \cdot \chi(\vecx, \vecy) ~~~~ \forall \vec{x},   \vec{y} \in \R_+^{d}$ and can
be used to design differentially private algorithms when the objective function has low $\chi$
sensitivity, according to the following lemma. The proof of this Lemma follows directly from the definitions of $\diver_{\infty}$ and
$S_{\chi}$.


\begin{lemma} \label{lem:privacyFromFtoA}
    Let $\vec{f} : \R_+^d \to \Delta_{d}$ be a soft maximum function, $w : \Domain^n \to \R_+$ be
  an objective function. If $\vec{f}$ is $L$-Lipschitz with respect to $\diver_{\infty}$ and 
  $\chi$, then $A_{\vecf}$ is $(L/S_{\chi}(w))$-differentially private.
\end{lemma}

\subsubsection{Application to Differentially Private Submodular Optimization} 
\label{sec:multiplicative:submodular}

  In differentially private maximization of submodular functions under cardinality constraints, we 
observe that if the input data set satisfies a mild assumption, then using power mechanism we 
achieve an asymptotically smaller error compare to the state of the art algorithm of  Mitrovic et
al. \cite{MitrovicBKK17}. 

\paragr{Submodular Functions.} Let $\Domain$ be a set of elements with $d = \abs{\Domain}$. A 
function $h : 2^{\Domain} \to \realsp$ is called submodular if 
$h(R \cup \{v\}) - h(R) \ge h(T \cup \{v\}) - h(T)$ for all $R \subseteq T \subseteq \Domain$ and
all $v \in \Domain \setminus T$.

\paragr{Monotone Functions.} A function $h : 2^{\Domain} \to \realsp$ is monotone if
$h(T) \ge h(R)$ for all $R \subseteq T \subseteq \Domain$.

\noindent Monotone and Submodular Maximization under Cardinality Constraints is
the optimization problem $\max_{R \subseteq \Domain, \abs{R} \le k} h(R)$. 
We use the Algorithm 1 of \cite{MitrovicBKK17}, where we replace the
exponential mechanism in the soft maximization step with the power mechanism. Let $S_{l, q}$  be the 
sensitivity of $h$ with respect to $q$-log-Euclidean quasi-metric and $\OPT$ be the optimal
value.

\begin{theorem} \label{thm:monotoneSubmodularCardinality}
    Let $h : 2^{\Domain} \to \R_+$ be a monotone and submodular function. Then, there exists an 
  efficient $(\eps, \delta)-$differentially private algorithm with output $R_k$ that achieves  multiplicative 
  approximation guarantee
  $\Exp[h(R_k)] \ge \left(1 - \exp\p{d^{\frac{\eps}{S_{l, \infty}(h) \p{\sqrt{k} + \sqrt{\log(1/\delta)}}} - 1}} \right) \OPT$.
\end{theorem}

  Even though it is not immediately clear if the above guarantee is better than the one 
of \cite{MitrovicBKK17}, we note  that the above result has only
a multiplicative approximation error. In contrast, the algorithm of \cite{MitrovicBKK17} has both 
multiplicative and additive error. In general, it is impossible to compare the two 
tradeoffs, 
because the tradeoff of 
\cite{MitrovicBKK17} is parameterized by $S_{\infty}$ sensitivity of $h$ whereas 
our tradeoff is parameterized by $S_{l, \infty}$ of $h$. Even though  there is no
a priori comparison between the two sensitivities, the following mild 
assumption allows us to compare them.

\begin{definition}[\textsc{$t$-Multiplicative Insensitivity}] \label{def:nicenessDefinition}
    A data-set $\Domain^n$ is $t$-\textit{multiplicative insensitive} for an objective
  function $w : \Domain^n \times [d] \to \realsp$ if for any two inputs
  $\matr{V}, \matr{V'} \in \Domain^n$ that differ only in one coordinate and for any $i \in [d]$, if
  $w(\matr{V'}, i) \le w(\matr{V}, i)$ it holds that
  $\frac{w(\matr{V'}, i)}{w(\matr{V}, i)} \ge 1 - \frac{1}{t} \frac{S_{\infty}(w)}{\OPT}$.
\end{definition}

  Based on the above definition, we prove that the error of the power mechanism, under the assumption of 
$O(1)$-multiplicative insensitivity, is asymptotically better than the error of the
exponential mechanism. This improvement is also observed in experiments with real world data as it is shown in 
Figure  \ref{fig:experiment-sub}. The missing proofs and a detailed explanation of 
results is in the~\refappendix{sec:submodular}.
  
\begin{corollary} \label{cor:monotoneSubmodularCardinalityWithAssumption}
    Assume the input data satisfy $O(1)$-multiplicative insensitivity. Let $T_{k}$ be the output of
  Algorithm 1 of \cite{MitrovicBKK17} using the exponential mechanism, then the approximation guarantee is
  \[ \Exp[h(T_k)] \ge \left(1 - 1/e \right) \OPT - O \left( k \cdot S_{\infty}(h) \log \abs{\Domain}/\eps \right) \]
  whereas if $R_k$ is the output when using the power mechanism, then the approximation guarantee
  is
  \[ \Exp[h(R_k)] \ge \left(1 - 1/e \right) \OPT - O \left( \sqrt{k} \cdot S_{\infty}(h) \log \abs{\Domain}/\eps \right). \]
\end{corollary}

\begin{figure}[t]
\vspace{-1cm}
\centering
  \begin{subfigure}[t]{0.4\textwidth}
    \includegraphics[width=0.95\textwidth]{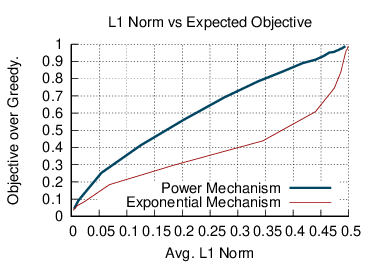}
    \caption{$\ell_1$ Distance}
    \label{fig:experiment-1-sub}
  \end{subfigure}
  ~
  \begin{subfigure}[t]{0.4\textwidth}
      \includegraphics[width=0.95\textwidth]{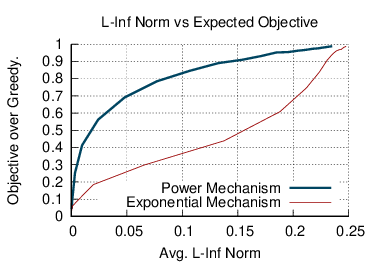}
      \caption{$\ell_\infty$ Distance}
      \label{fig:experiment-2-sub}
  \end{subfigure}
\caption{Smoothness vs utility in the submodular maximization with cardinality constraint $k=10$. The y-axis shows the ratio of the average objective to the (non-private) greedy algorithm. The x-axis represents the sensitivity to the manipulation test of the value of the first element selected. }\label{fig:experiment-sub}
\end{figure}

We validated these theoretical results with an empirical study we report fully in \refappendix{sec:experiments}. Here we briefly outline our results in Figure~\ref{fig:experiment-sub}, where we show improved objective vs sensitivity trade-offs for the power mechanism in an empirical data manipulation tests. In this experiments we manipulated randomly a submodular optimization instance, and measured how the output distribution of a differentially private soft-max (Power and Exponential mechanism with a given parameter) is affected by the manipulation (x-axis). In the y-axis we report the average objective obtained by the algorithm and parameter setting. The results in Figure~\ref{fig:experiment-sub} show that, for the same level of empirical sensitivity, the power mechanisms allows substantially improved results.

\subsection{Sparse Multi-class Classification}
\label{sec:sparse_classification}

  Sparsity, or in our language worst-case approximation guarantee, is relevant both in multiclass
classification and in designing attention mechanisms \cite{MartinsA16, LahaCAJSR18}. As illustrated
in Theorem \ref{thm:mainsoft-maxFunctionTV}, \plsm\ has small $\ell_q \to \ell_p$ 
smoothness for any $p, q$. In contrast, the mechanisms proposed in \cite{MartinsA16, LahaCAJSR18}
achieve much worse $\ell_q \to \ell_1$ smoothness as we can see below.
\begin{lemma} \label{lem:sparsemaxLowerBound}
    Let $h(\cdot) = \mathrm{sparsegen}\text{-}\mathrm{lin}(\cdot)$ be the generalization of $\mathrm{sparsemax}(\cdot)$
  function, then there exist $\vecx, \vecy \in \R^d$ such that 
  $\norm{h(\vecx) - h(\vecy)}_1 \ge \frac{1}{2} d^{1 - 1/q} \norm{\vecx - \vecy}_q$.
\end{lemma}
In contrast, \plsm\ achieves $\ell_q \to \ell_1$ smoothness of order 
$\min\{q + 1, \log(d)\}$.
Smoothness is preferred for gradient calculation in commonly adopted stochastic gradient descent algorithms. To illustrate this we define a loss function
with properties that are summarized in the following proposition. A detailed explanation of the
loss function and a proof of Proposition \ref{prop:lossFunction:properties} are presented in Appendix
\ref{sec:lossFunction}.

\begin{proposition} \label{prop:lossFunction:properties}
    The exists a loss function $L_{\Rec} : \R^d \times \Delta_{d - 1} \to \R_+$ 
  such that it holds:
  (1) $L_{\Rec}(\vecx; \vecq) \ge 0$, 
  (2) $L_{\Rec}(\vecx; \vecq) = 0 \Leftrightarrow \Rec^{\delta}(\vecx) = \vecq$,
  (3) $L_{\Rec}(\vecx; \vecq)$ is a convex function with respect to $\vecx$.
\end{proposition}

\section*{Acknowledgements}

MZ was supported by a Google Ph.D. Fellowship. The authors thank Thabo Samakhoana and Benjamin Grimmer for pointing out a technical gap in the proof of Theorem \ref{thm:totalVariationLowerBound}, which we fixed in this version.

\bibliographystyle{alpha}
\bibliography{ref}

\newcommand{\etalchar}[1]{$^{#1}$}
\begin{thebibliography}{MBKK17}

\bibitem[BBHM05]{BalcanBHM05}
M-F Balcan, Avrim Blum, Jason~D Hartline, and Yishay Mansour.
\newblock Mechanism design via machine learning.
\newblock In {\em Foundations of Computer Science, 2005. FOCS 2005. 46th Annual
  IEEE Symposium on}, pages 605--614. IEEE, 2005.

\bibitem[Bol68]{Boltzman}
Ludwig Boltzmann.
\newblock Studien uber das gleichgewicht der lebenden kraft.
\newblock {\em Wissenschafiliche Abhandlungen}, 1:49--96, 1868.

\bibitem[Bri90a]{Bridle}
John~S. Bridle.
\newblock Probabilistic interpretation of feedforward classification network
  outputs, with relationships to statistical pattern recognition.
\newblock In Fran{\c{c}}oise~Fogelman Souli{\'e} and Jeanny H{\'e}rault,
  editors, {\em Neurocomputing}, pages 227--236, Berlin, Heidelberg, 1990.
  Springer Berlin Heidelberg.

\bibitem[Bri90b]{NIPS1989_195}
John~S. Bridle.
\newblock Training stochastic model recognition algorithms as networks can lead
  to maximum mutual information estimation of parameters.
\newblock In D.~S. Touretzky, editor, {\em Advances in Neural Information
  Processing Systems 2}, pages 211--217, 1990.

\bibitem[CD17]{CaiD17}
Yang Cai and Constantinos Daskalakis.
\newblock Learning multi-item auctions with (or without) samples.
\newblock {\em arXiv preprint arXiv:1709.00228}, 2017.

\bibitem[CR14]{ColeR14}
Richard Cole and Tim Roughgarden.
\newblock The sample complexity of revenue maximization.
\newblock In {\em Proceedings of the forty-sixth annual ACM symposium on Theory
  of computing}, pages 243--252. ACM, 2014.

\bibitem[dBV15]{de2015exploration}
Alexandre de~Br{\'e}bisson and Pascal Vincent.
\newblock An exploration of softmax alternatives belonging to the spherical
  loss family.
\newblock {\em arXiv preprint arXiv:1511.05042}, 2015.

\bibitem[DHP16]{DevanurHP16}
Nikhil~R Devanur, Zhiyi Huang, and Christos-Alexandros Psomas.
\newblock The sample complexity of auctions with side information.
\newblock In {\em Proceedings of the forty-eighth annual ACM symposium on
  Theory of Computing}, pages 426--439. ACM, 2016.

\bibitem[DP09]{DrakakisP09}
Konstantinos Drakakis and BA~Pearlmutter.
\newblock On the calculation of the l2→ l1 induced matrix norm.
\newblock {\em International Journal of Algebra}, 3(5):231--240, 2009.

\bibitem[DR14]{dworkr14}
Cynthia Dwork and Aaron Roth.
\newblock The algorithmic foundations of differential privacy.
\newblock {\em Foundations and Trends® in Theoretical Computer Science},
  9(3–4):211--407, 2014.

\bibitem[DRV10]{DworkRV10}
Cynthia Dwork, Guy~N Rothblum, and Salil Vadhan.
\newblock Boosting and differential privacy.
\newblock In {\em 2010 IEEE 51st Annual Symposium on Foundations of Computer
  Science}, pages 51--60. IEEE, 2010.

\bibitem[DRY15]{DhangwatnotaiRY15}
Peerapong Dhangwatnotai, Tim Roughgarden, and Qiqi Yan.
\newblock Revenue maximization with a single sample.
\newblock {\em Games and Economic Behavior}, 91:318--333, 2015.

\bibitem[EMZ17]{SPAA17}
Alessandro Epasto, Vahab Mirrokni, and Morteza Zadimoghaddam.
\newblock Bicriteria distributed submodular maximization in a few rounds.
\newblock In {\em SPAA}. ACM, 2017.

\bibitem[GBC16]{Goodfellow-et-al-2016}
Ian Goodfellow, Yoshua Bengio, and Aaron Courville.
\newblock {\em Deep Learning}.
\newblock MIT Press, 2016.

\bibitem[Gib02]{gibbs1902}
J.W. Gibbs.
\newblock {\em Elementary Principles in Statistical Mechanics: Developed with
  Especial Reference to the Rational Foundations of Thermodynamics}.
\newblock C. Scribner's sons, 1902.

\bibitem[Gli52]{Glicksberg1952}
Irving~L Glicksberg.
\newblock A further generalization of the kakutani fixed point theorem, with
  application to nash equilibrium points.
\newblock {\em Proceedings of the American Mathematical Society},
  3(1):170--174, 1952.

\bibitem[HK12]{ExpMechDesign}
Zhiyi Huang and Sampath Kannan.
\newblock The exponential mechanism for social welfare: Private, truthful, and
  nearly optimal.
\newblock In {\em Proceedings of the 2012 IEEE 53rd Annual Symposium on
  Foundations of Computer Science}, FOCS ’12, page 140–149, USA, 2012. IEEE
  Computer Society.

\bibitem[HO10]{HendrickxO10}
Julien~M Hendrickx and Alex Olshevsky.
\newblock Matrix p-norms are np-hard to approximate if $p \neq 1, 2, \infty$.
\newblock {\em SIAM Journal on Matrix Analysis and Applications},
  31(5):2802--2812, 2010.

\bibitem[JJ94]{JordanJacobs}
Michael~I. Jordan and Robert~A. Jacobs.
\newblock Hierarchical mixtures of experts and the em algorithm.
\newblock {\em Neural Comput.}, 6(2):181–214, March 1994.

\bibitem[LCA{\etalchar{+}}18]{LahaCAJSR18}
Anirban Laha, Saneem~Ahmed Chemmengath, Priyanka Agrawal, Mitesh Khapra,
  Karthik Sankaranarayanan, and Harish~G Ramaswamy.
\newblock On controllable sparse alternatives to softmax.
\newblock In {\em Advances in Neural Information Processing Systems}, pages
  6422--6432, 2018.

\bibitem[Luc59]{Luce1959}
R.~Duncan Luce.
\newblock {\em Individual Choice Behavior: A Theoretical Analysis}.
\newblock John Wiley \& Sons, 1959.

\bibitem[MA16]{MartinsA16}
Andre Martins and Ramon Astudillo.
\newblock From softmax to sparsemax: A sparse model of attention and
  multi-label classification.
\newblock In {\em International Conference on Machine Learning}, pages
  1614--1623, 2016.

\bibitem[MB05]{morin2005hierarchical}
Frederic Morin and Yoshua Bengio.
\newblock Hierarchical probabilistic neural network language model.
\newblock In {\em Aistats}, volume~5, pages 246--252. Citeseer, 2005.

\bibitem[MBKK17]{MitrovicBKK17}
Marko Mitrovic, Mark Bun, Andreas Krause, and Amin Karbasi.
\newblock Differentially private submodular maximization: Data summarization in
  disguise.
\newblock In {\em ICML}, 2017.

\bibitem[MR15]{MorgensternR15}
Jamie~H Morgenstern and Tim Roughgarden.
\newblock On the pseudo-dimension of nearly optimal auctions.
\newblock In {\em Advances in Neural Information Processing Systems}, pages
  136--144, 2015.

\bibitem[MSC{\etalchar{+}}13]{mikolov2013distributed}
Tomas Mikolov, Ilya Sutskever, Kai Chen, Greg~S Corrado, and Jeff Dean.
\newblock Distributed representations of words and phrases and their
  compositionality.
\newblock In {\em Advances in neural information processing systems}, pages
  3111--3119, 2013.

\bibitem[MT07]{McSherryT07}
Frank McSherry and Kunal Talwar.
\newblock Mechanism design via differential privacy.
\newblock In {\em Foundations of Computer Science, 2007. FOCS'07. 48th Annual
  IEEE Symposium on}, pages 94--103. IEEE, 2007.

\bibitem[Mye81]{Myerson1981}
Roger~B Myerson.
\newblock Optimal auction design.
\newblock {\em Mathematics of operations research}, 6(1):58--73, 1981.

\bibitem[Roh00]{Rohn2000}
Ji{\v{r}}{\'\i} Rohn.
\newblock Computing the norm (oo,1) is np-hard.
\newblock {\em Linear and Multilinear Algebra}, 47(3):195--204, 2000.

\bibitem[RTCY12]{RoughgardenTY12}
Tim Roughgarden, Inbal Talgam-Cohen, and Qiqi Yan.
\newblock Supply-limiting mechanisms.
\newblock In {\em Proceedings of the 13th ACM Conference on Electronic
  Commerce}, pages 844--861. ACM, 2012.

\bibitem[SB18]{RLBook}
Richard~S. Sutton and Andrew~G. Barto.
\newblock {\em Reinforcement Learning: An Introduction}.
\newblock A Bradford Book, Cambridge, MA, USA, 2018.

\bibitem[THB86]{TitchmarshH1986}
Edward~Charles Titchmarsh and David~Rodney Heath-Brown.
\newblock {\em The theory of the Riemann zeta-function}.
\newblock Oxford University Press, 1986.

\bibitem[VDBB15]{vincent2015efficient}
Pascal Vincent, Alexandre De~Br{\'e}bisson, and Xavier Bouthillier.
\newblock Efficient exact gradient update for training deep networks with very
  large sparse targets.
\newblock In {\em Advances in Neural Information Processing Systems}, pages
  1108--1116, 2015.

\end{thebibliography}
\clearpage
\appendix

\section{Lower Bounds for the Exponential Mechanism} \label{app:expLB}

  In this section, we prove Theorems~\ref{thm:expLB} and \ref{thm:expNoWorstCase}.
  
\begin{proof}[Proof of Theorem~\ref{thm:expLB}.]
  Fix a soft-max function $\vecf : \R^d \to \Delta_d$ that is $\delta$-approximate. It is well
known that the R\'enyi Divergence of order $\alpha$ is a non-decreasing function of $\alpha$
for $\alpha \ge 1$. Hence it suffices to prove the statement of Theorem \ref{thm:expLB} for
$\alpha = 1$ where $\diver_{\alpha}$ become the KL-divergence $\KL$. Observe also that without
loss of generality we can assume that $\vecf$ is permutation invariant, i.e., for every permutation $\pi$ of $\{1,\ldots,d\}$ and every $\vecx\in\R^d$, $f(\pi(\vecx)) = \pi(f(\vecx))$, where $\pi(\vecx)$ denotes the vector $(x_{\pi(1)}, \ldots, x_{\pi(d)})$. If this is not the 
case then we can define the function $\vecf'$ which outputs the expectation of $\vecf$ over a
random permutation of the coordinates of $\vecx$. It is easy to see then that $\vecf'$ has the
same approximation and smoothness properties as $\vecf$ and is permutation invariant. Hence we assume that $\vecf$ is 
permutation invariant.

  Let $a \in \realsp$. We define the vector $\vec{x}_a = (a, a, \dots, a)^T$. For any $a$ 
because of the permutation invariance of $\vecf$ we have that 
$\vecf(\vecx_a) = (1/d, \dots, 1/d)$. We define the vector $\vecy^{(a, b)}$ to be equal to
$\vecx_a$ in all coordinates but $1$ and equal to $b > a$ at the $1$st coordinate. That is
\begin{align*}
                   & y^{(a, b)}_j = a ~~~ \text{ for } j \neq 1 \\
  \text{ and } ~~~ & y^{(a, b)}_1 = b.
\end{align*}
\noindent From the approximation guarantee at $\vecy^{(a, b)}$ we have that
\begin{align*}
  \norm{\vec{y}^{(a, b)}}_{\infty} - \langle \vec{f}\left( \vec{y}^{(a, b)} \right), \vec{y}^{(a, b)} \rangle & \le \delta \implies \\
  b - b f_1\left( \vec{y}^{(a, b)} \right) - a \left(1 - f_1\left( \vec{y}^{(a, b)} \right)\right) & \le \delta.
\end{align*}
\noindent Let $q = f_1\left( \vec{y}^{(a, b)} \right)$. Then we
have
\[ (b - a)(1 - q) \le \delta. \]
\noindent This implies
\begin{equation} \label{eq:app:expLB:probabilityLowerBound}
  q \ge 1 - \frac{\delta}{b - a}.
\end{equation}
\noindent Also observe that because of the permutation invariance of $\vecf$ it holds that 
$f_i(\vecy^{(a, b)}) = (1 - q)/(d - 1)$ for any $i > 1$. Now we bound the
KL-divergence of $\vecf$ when applied to the vectors $\vec{x}_a$ and $\vec{y}^{(a, b)}$:
\begin{align*}
  \renyi{\mathrm{KL}}{\vec{f}\left(\vec{y}^{(a, b)}\right)}{\vec{f}(\vec{x}_a)} & = 
  \sum_{i = 1}^d f_i(\vecy^{(a, b)}) \log \p{\frac{f_i(\vecy^{(a, b)})}{f_i(\vecx_a)}} \\
  & = q \log\p{d q} + (1 - q) \log\p{(1 - q) \frac{d}{d - 1}} \\
  & \ge q \log\p{d} - 1,
  \intertext{where the last inequality follows from the fact that the binary entropy function 
  $H(q) = - q \log(q) - (1 - q) \log(1 - q)$ is upper bounded by $1$ and the fact that 
  $\log(d) \ge \log(d - 1)$. Using also \ref{eq:app:expLB:probabilityLowerBound} we get that}
  \renyi{\mathrm{KL}}{\vec{f}\left(\vec{y}^{(a, b)}\right)}{\vec{f}(\vec{x}_a)} &
  \ge \p{1 - \frac{\delta}{b - a}} \log\p{d} - 1.
  \intertext{If we now set $b - a = 2 \delta$ then we get $\norm{\vecy^{(a, b)} - \vecx_a}_p = 2\delta$ and}
  \renyi{\mathrm{KL}}{\vec{f}\left(\vec{y}^{(a, b)}\right)}{\vec{f}(\vec{x}_a)} &
  \ge \frac{1}{2} \log\p{d} - 1.
  \intertext{Therefore,}
  \frac{\renyi{\mathrm{KL}}{\vec{f}\left(\vec{y}^{(a, b)}\right)}{\vec{f}(\vec{x}_a)}}{\norm{\vecy^{(a, b)} - \vecx_a}_p} &
  \ge \frac{\log\p{d} - 2}{4 \delta}
\end{align*}
and the theorem follows.
\end{proof}

\begin{proof}[Proof of Theorem~\ref{thm:expNoWorstCase}.]
  Let $\delta > 0$ and for the sake of contradiction assume that there exists a soft-max
function $\vecf$ that is both $\delta$-approximate in the worst-case and satisfies 
$(\ell_p, \diver_{\alpha})$-Lischitzness. We define $\vecx = (2 \delta, 0, 0, \dots, 0)$ and
$\vecy = (0, 2 \delta, 0, \dots, 0)$ from the worst-case approximation guarantees of $\vecf$ 
we have that $\vecf(\vecx) = (1, 0, \dots, 0)$, whereas $\vecf(\vecy) = (0, 1, 0, \dots, 0)$.
It is easy to see that for any $\alpha \ge 1$ it holds that 
$\renyi{\alpha}{\vecf(\vecx)}{\vecf(\vecx)} = \infty$ but 
$\norm{\vecf(\vecx) - \vecf(\vecy)}_p \le 2$. The later contradicts the 
$(\ell_p, \diver_{\alpha})$-Lipschitzness of $\vecf$ and hence the theorem follows.
\end{proof}
\section{The Construction of \plsm} 
\label{app:plsm}

We first give an intuitive explanation of the proof of the construction. 
One notion that will be useful for this purpose in the following.
  
\paragr{Vector and Matrix Norms.} We define the 
\textit{$({\alpha, \beta})$-subordinate norm} of a matrix  
$\matr{A} \in \reals^{d \times \ell}$ to be
\[ \norm{\matr{A}}_{\alpha, \beta} = \max_{\vec{x} \in \reals^\ell, \vec{x} \neq 0} \norm{\matr{A} \vec{x}}_{\beta}/\norm{\vec{x}}_{\alpha}. \] 
The computation of $\norm{\matr{A}}_{\alpha, \beta}$ is in general NP-hard and even hard to
approximate, see \cite{Rohn2000, HendrickxO10}.
\medskip

\paragr{Notation.} We use $\matr{E}_{i, j}$ to refer to the all zero matrix with one
$1$ at the $(i, j)$ entry.

  The construction of \plsm\ begins with the observation that for any 
$\vecg : \R^d \to \R^d$ and any $p, q \ge 1$, it holds that
\[ \norm{\vecg(\vecx) - \vecg(\vecy)}_q \le \left( \max_{\vec{\xi} \in \R^d} \norm{\matJ_{\vecg}(\vec{\xi})}_{p, q} \right) \norm{\vecx - \vecy}_p \]
\noindent where $\matJ_{\vecg}(\vec{\xi})$ is the Jacobian matrix of $\vecg$ at the point 
$\vec{\xi} \in \R^d$. Hence our goal is to construct a function $\vecg$ that does not violate
the worst-case approximation conditions and for which we can also bound 
$\norm{\matJ_{\vecg}(\vec{\xi})}_{p, q}$. To achieve this we carefully analyze the
approximation conditions. Based on them we split the space $\R^d$ into small convex polytopes
$P_i$ such that in each $P_i$, the approximation conditions do not change. Since, as we will
see, the approximation condition is a linear condition, we choose our function $\vecg$ in
$P_i$ to be a linear function that satisfies the approximation condition inside the polytope
$P_i$. Then we have to make sure that on the boundaries of $P_i$ the function is continuous
and that the $({p, q})$-subordinate norm of the matrices that we used in each $P_i$ is bounded by some
constant.

One important observation is that in each $P_i$, if some of the $[d]$ alternatives have low values,
the approximation constraint imposes that we cannot use at all any of these alternatives.
Hence the dimension of $P_i$ effectively becomes less than $d$. 
In these cases, we reduce the construction in $P_i$ to a smaller dimensional construction that is solved inductively. 
We express this inductive argument as a recursive
relation over the matrices that is stated in Lemma~\ref{lem:soft-maxMatricesContinuity1}. 
Finally, one important theorem that enables us to prove
a precise bound on $\norm{\matJ_{\vecg}(\vec{\xi})}_{p, 1}$ is Theorem
\ref{thm:drakakisGeneral}. This is a generalization of Theorem 1 of
\cite{DrakakisP09} which might be of independent interest.

  Now that we described the high level idea of our construction, we dive in to the technical
details. The function $\vecf$ that we are going to construct is a piecewise linear function.
So we first define the notion of a piecewise linear function in $d$ dimensions.

\begin{definition}[\textsc{Piecewise Linear Functions}]
    A function $\vecf : \R^{d} \to \R^{d}$ is \textit{piecewise linear} if there exist a
  finite partition $\mathcal{P}_{\vecf}= \{P_1, \dots, P_L\}$ of $\R^d$ such that $P_i$ is a
  convex polytope, for any $i$ and any $\vecx \in P_i$ there exists a unique matrix 
  $\matA_i \in \R^{d \times d}$ and a unique vector $\vecb_i \in \R^{d}$ such that
  \[ \vecf(\vecx) = \matA_i \vecx + \vecb_i. \]
  \noindent We use $\calA_{\vecf}$ to refer to the set of matrices 
  $\{\matA_1, \dots, \matA_L\}$.
\end{definition}

  Our construction proceeds in the following steps:
\begin{enumerate}
  \item define the partition $\mathcal{P}_{\vec{f}}$ of $\R^d$, the matrix $\matr{A}_i$,
  and vector $\matr{b}_i$ that we use for every $P_i \in \mathcal{P}_{\vec{f}}$,
  \item describe the set $\mathcal{A}_{\vec{f}}$ and its properties,
  \item prove that the defined $\vec{f}$ is continuous on the boundaries of $P_i$'s,
  \item prove that it has small absolute approximation loss, and
  \item prove that $\norm{\matr{A}_i}_{p, 1}$ is small and hence using Lemma 
  \ref{lem:LischitzTVFromNorms} conclude that $\vec{f}$ is has small Lipschitz constant.
\end{enumerate}
  For simplicity of the proof we will use $\vecf$ to refer to $\plsm^{\delta}$ within the
scope of this section. 

\subsection{Piecewise linear functions}
  For piecewise linear functions $\vecf$, we use the following lemma to establish the Lipschitz property.

\begin{lemma} \label{lem:LischitzTVFromNorms}
    Let $\vec{f} : \R^d \to \R^d$ be a continuous and piecewise linear function and let
  $p, q \ge 1$, then 
  \[ \norm{\vecf(\vecx) - \vecf(\vecy)}_q \le \left( \max_{\matA \in \calA_{\vecf}} \norm{\matA}_{p, q} \right) \cdot \norm{\vecx - \vecy}_p ~~~~ \forall \vecx, \vecy \in \R^d \]
\end{lemma}

\begin{proof}
    We first prove the single variable case, that is, we prove that for any continuous piecewise linear
  function $\vecg : \R \to \R^d$ and if 
  $c = \max_{\matA \in \calA_{\vecg}} \norm{\matA}_{p, q}$ then for any $x, y \in \R$
  \[ \norm{\vecg(x) - \vecg(y)}_q \le c \abs{x - y}. \]
  \noindent Without loss of generality assume that $x > y$. Since $\vecg$ is piecewise linear,
  we have a sequence $y=x_1<x_2< \cdots< x_L=x$ such that for any $z \in [x_i, x_{i + 1}]$ : 
  $\vecg(z) = \veca_i z + \vecb_i$ for some $\veca_i, \vecb_i \in \R^d$. Also notice that since $\veca_i$ is a vector,
  by definition of subordinate norms, $\norm{\veca_i}_{p, q} = \norm{\veca_i}_q$. Now because of the
  continuity of $\vecg$ 
  \begin{align*}
    \norm{\vecg(x) - \vecg(y)}_q & \le \sum_{i = 1}^{L-1} \norm{\vecg(x_{i + 1}) - \vecg(x_i)}_q = \sum_{i = 1}^{L-1} \norm{\veca_i (x_{i + 1} - x_i)}_q = \sum_{i = 1}^{L-1} \norm{\veca_i}_q (x_{i + 1} - x_i) \\
                         & \le c \left( \sum_{i = 1}^{L-1} (x_{i + 1} - x_i) \right) = c (x - y).
  \end{align*}

    For the general case, let $c = \max_{\matA \in \calA_{\vecf}} \norm{\matA}_{p, q}$ and 
  $\vecx, \vecy \in \R^{d}$. We define the following function $\vech : [0, 1] \to \R^{d}$
  which is easy to verify that is also continuous and piecewise linear:
  \[ \vech(t) = \vecf\left(t \vecx + (1 - t) \vecy\right). \]
  There exists a sequence $0=t_1< t_2<\cdots< t_L=1$, such that for every $i$, the 
  function $\vecf$ has a linear form $\vecf(\vecu)=A_i\vecu+b_i$ on the set $\{t \vecx + (1 - t) \vecy:~~t \in [t_i, t_{i + 1}]\}$.
Therefore, for every $t\in [t_i, t_{i + 1}]$, by the definition of $\vech$,
$$\vech(t) = \matA_i (t \vecx + (1 - t) \vecy) + \vecb_i=
\matA_i (\vecx - \vecy) t + \vecb_i + \matA_i \vecy.$$
Therefore, on $t\in [t_i, t_{i + 1}]$, the function $\vech$ has the linear form $\vech(t) = \vecv_i t + \vecw_i$ for 
$\vecv_i = \matA_i (\vecx - \vecy)$ and $\vecw_i = \matA_i \vecy + \vecb_i$.
  Hence by the definition of the subordinate matrix norm we have that
  \[ \norm{\vecv_i}_q = \norm{\matA_i (\vecx - \vecy)}_q \le \norm{\matA_i}_{p, q} \norm{\vecx - \vecy}_{p} \le c \norm{\vecx - \vecy}_{p}. \]
  \noindent Since $i$ was arbitrary we have that 
  $c' = \max_{\matA \in \calA_h} \norm{\matA}_{p, q} \le c \norm{\vecx - \vecy}_{p}$. Finally
  using the statement of the lemma for the single variable case that we already proved, we
  have that 
  \[ \norm{\vecf(\vecx) - \vecf(\vecy)}_q = \norm{\vec{h}(1) - \vec{h}(0)}_q \le c' (1 - 0) \le c \norm{\vec{x} - \vec{y}}_{p}. \]
\end{proof}

\subsection{Properties of the Soft-Max Matrices} \label{app:soft-maxmat}

Recall the definition of the soft max matrices in Section~\ref{sec:plsm}.

\begin{definition}[\textsc{Soft-Max Matrices}] \label{def:app:soft-maxMatrices} 
    The soft max matrix $\matr{SM}_{(k, d)} = (m_{i j}) \in \reals^{d \times d}$ with
  parameters $k$, $d$ is defined as follows
  \begin{align} \label{eq:soft-maxMatricesDefinition}
    m_{1 1} & = \frac{k - 1}{k}        & \\
    m_{i i} & = \frac{1}{i}            & ~~~~ \forall i \in [2, k] \\
    m_{i 1} & = - \frac{1}{k}          & ~~~~ \forall i \in [2, k] \\
    m_{i j} & = \frac{1}{j} - \frac{1}{j - 1}  & ~~~~ \forall j > i, j \in [2, k] \\
    m_{i j} & = 0                      & ~~~~ \forall i, j \text{ s.t. } (i \in [k + 1, d]) \vee (j \in [k + 1, d]).
  \end{align}
 \end{definition}

\noindent Schematically we have
\begin{align*}
  \matr{SM}_{(k, d)} = \begin{pmatrix}
                         \frac{k - 1}{k} & -\frac{1}{2} & -\frac{1}{6} & \cdots & - \frac{1}{k (k - 1)} & 0 & \cdots & 0 \\
                         - \frac{1}{k}   & \frac{1}{2}  & -\frac{1}{6} & \cdots & - \frac{1}{k (k - 1)} & 0 & \cdots & 0 \\
                         - \frac{1}{k}   & 0            & \frac{1}{3}  & \cdots & - \frac{1}{k (k - 1)} & 0 & \cdots & 0 \\
                         - \frac{1}{k}   & 0            & 0            & \cdots & - \frac{1}{k (k - 1)} & 0 & \cdots & 0 \\
                         \vdots          & \vdots       & \vdots       & \ddots & \vdots                & \vdots & \ddots & \vdots \\
                         - \frac{1}{k}   & 0            & 0            & \cdots & \frac{1}{k}           & 0 & \cdots & 0 \\
                         0               & 0            & 0            & \cdots & 0                     & 0 & \cdots & 0 \\
                         \vdots          & \vdots       & \vdots       & \ddots & \vdots                & \vdots & \ddots & \vdots \\
                         0               & 0            & 0            & \cdots & 0                     & 0 & \cdots & 0 \\
                       \end{pmatrix}
\end{align*}

  We also define the columns and the rows of the soft max matrices as follows
\begin{align} \label{eq:soft-maxMatricesColumns}
  \matr{SM}_{(k, d)} = \begin{pmatrix}
                         \mid                 & \mid                  &        & \mid                & \mid      &          & \mid    \\
                         \vec{m}^{(k, d)}_1   & \vec{m}^{(k, d)}_2    & \cdots & \vec{m}^{(k, d)}_k  & \vec{0}   & \cdots   & \vec{0} \\
                         \mid                 & \mid                  &        & \mid                & \mid      &          & \mid    \\
                       \end{pmatrix}
\end{align}

\begin{align} \label{eq:soft-maxMatricesRows}
  \matr{SM}_{(k, d)} = \begin{pmatrix}
                         \text{---} & \left( \vec{s}^{(k, d)}_1 \right)^T & \text{---} \\
                         \text{---} & \left( \vec{s}^{(k, d)}_2 \right)^T & \text{---} \\
                                    & \vdots                              &            \\
                         \text{---} & \left( \vec{s}^{(k, d)}_k\right)^T  & \text{---} \\
                         \text{---} & \vec{0}                             & \text{---} \\
                                    & \vdots                              &            \\
                         \text{---} & \vec{0}                             & \text{---} \\
                       \end{pmatrix}
\end{align}

\noindent Below are some examples for $d = 4$.

\begin{align*}
   &\matr{SM}_{(1, 4)} = \begin{pmatrix}
                           0             & 0             & 0             & 0            \\
                           0             & 0             & 0             & 0            \\
                           0             & 0             & 0             & 0            \\
                           0             & 0             & 0             & 0            \\
                         \end{pmatrix}
   &\matr{SM}_{(2, 4)} = \begin{pmatrix}
                           \frac{1}{2}   & -\frac{1}{2}  & 0             & 0            \\
                           -\frac{1}{2}  & \frac{1}{2}   & 0             & 0            \\
                           0             & 0             & 0             & 0            \\
                           0             & 0             & 0             & 0            \\
                         \end{pmatrix} ~~~~~~ \\
   &\matr{SM}_{(3, 4)} = \begin{pmatrix}
                           \frac{2}{3}   & -\frac{1}{2}  & -\frac{1}{6}  & 0            \\
                           -\frac{1}{3}  & \frac{1}{2}   & -\frac{1}{6}  & 0            \\
                           -\frac{1}{3}  & 0             & \frac{1}{3}   & 0            \\
                           0             & 0             & 0             & 0            \\
                         \end{pmatrix}
   &\matr{SM}_{(4, 4)} = \begin{pmatrix}
                           \frac{3}{4}   & -\frac{1}{2}  & -\frac{1}{6}  & -\frac{1}{12} \\
                           -\frac{1}{4}  & \frac{1}{2}   & -\frac{1}{6}  & -\frac{1}{12} \\
                           -\frac{1}{4}  & 0             & \frac{1}{3}   & -\frac{1}{12} \\
                           -\frac{1}{4}  & 0             & 0             & \frac{1}{4}   \\
                         \end{pmatrix}
\end{align*}

  Now we prove some properties of the soft max matrices, that will help us later prove the 
continuity and the smoothness of \plsm.

\begin{lemma} \label{lem:soft-maxMatricesContinuity1}
    For any $d \in \nats$ and $k \in [d]$ the following recursive relation holds
  \[ \matr{SM}_{(k - 1, d)} = \matr{SM}_{(k, d)} \left( \matr{I} + \matr{E}_{k, 1} - \matr{E}_{k, k} \right). \]
\end{lemma}

\begin{proof}
  From \eqref{eq:soft-maxMatricesColumns} we have that
\begin{align*}
    \matr{SM}_{(k, d)} \left( \matr{I} + \matr{E}_{k, 1} - \matr{E}_{k, k} \right) =
                       \begin{pmatrix}
                         \mid                                      & \mid               &        & \mid                      & \mid      &          & \mid    \\
                           \vec{m}^{(k, d)}_1 + \vec{m}^{(k, d)}_k & \vec{m}^{(k, d)}_2 & \cdots & \vec{m}^{(k, d)}_{k - 1}  & \vec{0}   & \cdots   & \vec{0} \\
                         \mid                                      & \mid               &        & \mid                      & \mid      &          & \mid    \\
                       \end{pmatrix}.
\end{align*}

\noindent We now observe by the definition of the soft max matrices that for any $d \in \nats$, $k, k' \in [d]$
and $j \in [2, \min\{k, k'\}]$ it holds that $\vec{m}^{(k, d)}_j = \vec{m}^{(k', d)}_j$. Hence we only have to
prove that
\[ \vec{m}^{(k - 1, d)}_1 = \vec{m}^{(k, d)}_1 + \vec{m}^{(k, d)}_k \]
\noindent and the lemma follows. For this we have that
\[ m^{(k, d)}_{1 1} + m^{(k, d)}_{1 k} = \frac{k - 1}{k} - \frac{1}{k (k - 1)} = \frac{(k - 1)^2 - 1}{k (k - 1)} = \frac{k - 2}{k - 1} = m^{(k - 1, d)}_{1 1} \]
\noindent also for $i \in [2, k - 1]$ we have that
\[ m^{(k, d)}_{i 1} + m^{(k, d)}_{i k} = - \frac{1}{k} - \frac{1}{k (k - 1)} = - \frac{1}{k - 1} = m^{(k - 1, d)}_{i 1} \]
\noindent and finally
\[ m^{(k, d)}_{k 1} + m^{(k, d)}_{k k} = - \frac{1}{k} + \frac{1}{k} = 0 = m^{(k - 1, d)}_{k 1} \]
\noindent and the lemma follows.
\end{proof}

\begin{lemma} \label{lem:soft-maxMatricesContinuity2}
    Let $r, t \in [d]$ with $r < t$ and $\vecx \in \reals^d$ be a vector with the
  property that $x_i = x_j = x$ for any $i, j \in [r, t]$ then the vector 
  $\vecy \in \reals^d$ with
  \[ \vecy = \matr{SM}_{(k, d)} \vecx \]
  \noindent has also the property $y_i = y_j$ for any $i, j \in [r, t]$.
\end{lemma}

\begin{proof}
    From \eqref{eq:soft-maxMatricesRows} we have that
  \[ \vecy = \matr{SM}_{(k, d)} \vecx =
                       \begin{pmatrix}
                         \vec{s}_1^T \vecx \\
                         \vec{s}_2^T \vecx \\
                         \vdots \\
                         \vec{s}_k^T \vecx \\
                         0 \\
                         \vdots  \\
                         0 \\
                       \end{pmatrix}
   \]

  \noindent where for simplicity we dropped the indicators $(k, d)$ from the row vectors 
  $\vec{s}_i$ since we keep $k$, $d$ constant through the proof. Therefore we have that
  \[ \begin{pmatrix}
       y_r \\
       y_{r + 1} \\
       \vdots \\
       y_t
     \end{pmatrix} =
     \begin{pmatrix}
       \sum_{j = 1}^{r - 1} s_{r j} x_j + \left( \sum_{j = r}^t s_{r j} \right) x + \sum_{j = t + 1}^d s_{r j} x_j \\
       \sum_{j = 1}^{r - 1} s_{(r + 1) j} x_j + \left( \sum_{j = r}^t s_{(r + 1) j} \right) x + \sum_{j = t + 1}^d s_{(r + 1) j} x_j \\
     \vdots \\
     \sum_{j = 1}^{r - 1} s_{t j} x_j + \left( \sum_{j = r}^t s_{t j} \right) x + \sum_{j = t + 1}^d s_{t j} x_j
   \end{pmatrix}
\]

\noindent but by the definition of the soft max matrices we can easily see that for any $i, i' \in [r, t]$ and $j < r$
or $j > t$ it holds that $s_{i j} = s_{i' j}$. This observation together with the above calculations imply that it suffices
to prove that for any $i, i' \in [r, t]$ it holds that
\begin{align} \label{eq:proofsoft-maxMatricesContinuity2_1}
  \sum_{j = r}^t s_{i j} = \sum_{j = r}^t s_{i' j}
\end{align}
\noindent also because of the symmetry of the zero entries of soft max matrices for $t > k$ it suffices to prove this statement
for $t \le k$. We also consider two case $r = 1$ and $r > 1$.
\smallskip

\paragr{$\mathbf{r = 1}$.} For $i = 1$ we have that
\[ \sum_{j = r}^t s_{1 j} = s_{1 1} + \sum_{j = 2}^t s_{1 j} = \frac{k - 1}{k} - \sum_{j = 2}^t \frac{1}{j (j - 1)} \]
\noindent and using the following relation
\begin{align} \label{eq:seriesComputation1}
  \sum_{j = n}^m \frac{1}{j (j - 1)} = \sum_{j = n}^m \left( \frac{1}{j - 1} - \frac{1}{j} \right) = \frac{1}{m - 1} - \frac{1}{n}
\end{align}
\noindent we get that
\[ \sum_{j = r}^t s_{1 j} = \frac{k - 1}{k} - \left(1 - \frac{1}{t}\right) = \frac{1}{t} - \frac{1}{k}. \]

  For $i > 1$ we have that
  \[ \sum_{j = r}^t s_{i j} = s_{i 1} + s_{i i} + \sum_{j = i + 1}^t s_{i j} = - \frac{1}{k} + \frac{1}{i} - \sum_{j = i + 1}^t \frac{1}{j (j - 1)} \stackrel{\eqref{eq:seriesComputation1}}{=} = - \frac{1}{k} + \frac{1}{i} - \left( \frac{1}{i} - \frac{1}{t} \right) = \frac{1}{t} - \frac{1}{k}. \]
\noindent Hence the sum $\sum_{j = 1}^t s_{i j}$ does not depend on $i$ and the property \eqref{eq:proofsoft-maxMatricesContinuity2_1}
holds for $r = 1$.
\smallskip

\paragr{$\mathbf{r > 1}$.} For any $i \in [r, t]$ we have that
\[ \sum_{j = r}^t s_{i j} = s_{i i} + \sum_{j = i + 1}^t s_{i j} = \frac{1}{i} - \sum_{j = i + 1}^t \frac{1}{j (j - 1)} \stackrel{\eqref{eq:seriesComputation1}}{=} \frac{1}{i} - \left( \frac{1}{i} - \frac{1}{t} \right) = \frac{1}{t} \]
\noindent and again we observe that the sum $\sum_{j = r}^t s_{i j}$ does not depend on $i$ and the property \eqref{eq:proofsoft-maxMatricesContinuity2_1}
holds for any $r > 1$, $r \le t$. This implies $y_r = \cdots = y_t$ and the lemma follows.
\end{proof}

  Finally our goal is to bound $\norm{\matr{SM}_{(k, d)}}_{p, q}$ for any 
$p, q \in [1, \infty]$. Before that we give a proof of a general property of the subordinate
norm $\norm{\cdot}_{p, 1}$. This corresponds to the following generalization of Theorem 1 in
\cite{DrakakisP09}. Drakakis and Pearlmutter~\cite{DrakakisP09} only state the result for the
$\norm{\cdot}_{2, 1}$ norm although their proof generalizes.

\begin{theorem}[Generalization of Theorem 1 \cite{DrakakisP09}] \label{thm:drakakisGeneral}
    Let $\matr{A} \in \reals^{t \times d}$ and $p \in 2 \N_+$, then
  \[ \norm{\matr{A}}_{p, 1} = \max_{\vec{s} \in \{-1, 1\}^t} \norm{\vec{s}^T \matr{A}}_{r} \text{ ~~ where ~} r = \frac{p}{p - 1}. \]
  \noindent In particular the $\ell_r$ norm is the dual norm of the $\ell_p$ norm.
\end{theorem}

\begin{proof}[Proof of Theorem \ref{thm:drakakisGeneral}.]
    Let $\vec{a}_i^T$ be the $i$ th row of the matrix $\matr{A}$. By the definition of the
  subordinate norm we have that
  \[ \norm{\matr{A}}_{p, 1} = \max_{\vec{x} \in \reals^d, \norm{\vec{x}}_p = 1} \norm{\matr{A} \vec{x}}_1. \]
  We first prove that the maximum of the above optimization problem lies in a region of the space
  where $\vec{a}_i^T \vec{x} \neq 0$ for all $i \in [t]$. This implies that we can find the maximum
  in a subspace of the space where both the objective and the constraint are differentiable and
  hence we can use first order conditions to determine the maximum. This is described in the
  following claim.

  \begin{claim} \label{clm:thm:drakakisGeneral1}
    Let $p \ge 2$, and
    \[ \vec{x} = \argmax_{\vec{y} \in \reals^d, \norm{\vec{y}}_p = 1} \norm{\matr{A} \vec{y}}_1 \]
    \noindent then for every $i \in [t]$ such that $\vec{a}^T_i \neq \vec{0}$ it holds that $\vec{a}_i^T \vec{x} \neq 0$.
  \end{claim}
  \begin{proof}
      We prove this claim by contradiction. Let's assume without loss of generality 
    that all the rows of $\matr{A}$ are non-zero and that the we rearrange the rows 
    so that for $i = 1, \dots, \ell$ its true that $\veca^T_i \vecx = 0$, where 
    $\ell \in [t]$. Then we define the vector $\vecz$ as
    \[ \vec{z} = \frac{\vec{x} + \eta \vec{a}_1}{\norm{\vec{x} + \eta \vec{a}_1}_p} \] 
    \noindent with $\eta$ to be determined later in that proof but such that can be 
    either positive or negative and is small enough so that
    $\sign(\vec{a}_i^T \vec{x}) = \sign(\vec{a}_i^T \vec{z})$. We define the 
    following real valued function $h : \reals \to \reals$ as 
    $h(\eta) = 1/\norm{\vec{x} + \eta \vec{a}_1}_{p}$. Since $p \ge 2$, it is easy to
    see that the absolute value of the second derivative of $h$ for $\eta$ in the
    interval $[-1, 1]$ are bounded. Hence by Taylor's theorem we have that 
    \[ h(\eta) = h(0) + h'(0) \eta + O( \eta^2 ). \]
    \noindent By simple calculations it is also easy to see that $h(0) = 1$ and 
    $h'(0) = \sum_{i = 1}^t a_{1 i} x_i^{p - 1}$. Let also 
    $s_i = \sign(\vec{a}_i^T \vec{x})$. This implies 
    \begin{align*}
      \sum_{i = 1}^t \abs{\vec{a}_i^T \vec{z}} 
        & = \left( \sum_{i = 1}^t \abs{\vec{a}_i^T \vec{x}} + \abs{\eta} \sum_{i = 1}^\ell \abs{\vec{a}_i^T \vec{a}_1} + \eta \sum_{i = \ell + 1}^t s_i \vec{a}_i^T \vec{a}_1 \right) \left( h(0) + h'(0) \eta + O(\eta^2) \right) \\
        & = \sum_{i = 1}^t \abs{\vec{a}_i^T \vec{x}} + \abs{\eta} \sum_{i = 1}^\ell \abs{\vec{a}_i^T \vec{a}_1} + \left( \sum_{i = \ell + 1}^t s_i \vec{a}_i^T \vec{a}_1 + \sum_{i = 1}^t \abs{\vec{a}_i^T \vec{x}} h'(0) \right) \eta + O(\eta^2) \\
        & = \sum_{i = 1}^t \abs{\vec{a}_i^T \vec{x}} + C_1 \abs{\eta} + C_2 \eta + O(\eta^2).
    \end{align*}
    \noindent Since we have assumed that $\vec{a}_1 \neq \vec{0}$ this implies
    $C_1 > 0$. Also choosing the sign of $\eta$ to be equal to the sign of $C_2$ we
    have that $C_2 \eta \ge 0$. Finally we can make $\eta$ small enough so that 
    $C_1 \abs{\eta} + C_2 \eta + O(\eta^2) > 0$ and hence
    $\sum_{i = 1}^t \abs{\vec{a}_i^T \vec{z}} > \sum_{i = 1}^t \abs{\vec{a}_i^T \vec{x}}$ 
    which contradicts the assumption that $\vec{x}$ was the maximum and the claim
    follows.
  \end{proof}

  \smallskip
    Using Claim \ref{clm:thm:drakakisGeneral1} we can see that the maximum of the program 
  $\left( \max_{\vec{x} \in \reals^d, \norm{\vec{x}}_p = 1} \norm{\matr{A} \vec{x}}_1 \right)$ is
  achieved for a vector that belongs to an open subset of the space where both the constraint and
  the objective function are differentiable. Notice that the differentiability of the constraint 
  follows from the fact that $p$ is an even number.

  \smallskip
    Using Langragian multipliers we can find the solution to this optimization problem using first
  order conditions on the following function 
  \[ g(\vec{x}, \lambda) = \sum_{i = 1}^t \abs{\sum_{j = 1}^d a_{i j} x_j} + \lambda \left( \norm{x}_p - 1 \right) \]
  \noindent which using the definition $s_i = \sign(\vec{a}_i^T \vec{x})$ takes the form
  \[ g(\vec{x}, \lambda) = \sum_{i = 1}^t s_i \sum_{j = 1}^d a_{i j} x_j + \lambda \left( \norm{x}_p - 1 \right). \]
  \noindent We now compute the partial derivative of $g$ with respect to $x_k$ for some $k \in [d]$.
  \[ \frac{\partial g}{\partial x_k} = \sum_{i = 1}^t s_i a_{i k} + \lambda \frac{x_k^{p - 1}}{\norm{\vec{x}}_p^{p - 1}} = \sum_{i = 1}^t s_i a_{i k} + \lambda x_k^{p - 1} \]
  \noindent hence $\frac{\partial g}{\partial x_k} = 0$ implies
  \begin{equation} \label{eq:proof:thm:drakakisGeneral1}
    x_k = - \frac{1}{\lambda^{1/(p - 1)}} \left( \sum_{i = 1}^t s_i a_{i k}\right)^{1/(p - 1)}
  \end{equation}
  \noindent and therefore
  \[ \norm{\vec{x}}_p = \frac{1}{\abs{\lambda}^{1/(p - 1)}} \norm{\vec{s}^T \matr{A}}_{p/(p - 1)}^{1/(p - 1)}. \]
  \noindent From the constraint $\frac{\partial g}{\partial \lambda} = 0$ we get that
  \[ \abs{\lambda} = \norm{\vec{s}^T \matr{A}}_{p/(p - 1)}. \]
  \noindent Using \eqref{eq:proof:thm:drakakisGeneral1} and the definition of the function $g$ we 
  have that
  \begin{align*}
    g(\vec{x}, \lambda) & = \sum_{i = 1}^t s_i \sum_{j = 1}^d a_{i j} x_j = \sum_{j = 1}^d \left( \sum_{i = 1}^t s_i a_{i j} \right) x_j \\
                      & \stackrel{\eqref{eq:proof:thm:drakakisGeneral1}}{=} \sum_{j = 1}^d \left( - \lambda x_j^{p - 1} \right) x_j \\
                      & = - \lambda \sum_{j = 1}^d x_j^{p} = \norm{\vec{s}^T \matr{A}}_{r}
  \end{align*}
  \noindent where $r = \frac{p}{p - 1}$, and the theorem follows.
\end{proof}

\begin{lemma} \label{lem:sofmaxMatricesNormBound}
  For any $d \in \nats$, $k \in [d]$ and $p, q \in [1, \infty]$ we have that
  \[ \norm{\matr{SM}_{(k, d)}}_{p, q} \le 2 \min\set{p + 1, \frac{q}{q - 1}, \log(k)}. \]
\end{lemma}

\begin{proof}
    It is easy to see from the definition that 
  $\norm{\matr{SM}_{(k, d)}}_{p, q} = \norm{\matr{SM}_{(k, k)}}_{p, q}$. Hence we can restrict our
  attention to the matrices $\matr{SM}_{(k, k)}$ which for simplicity we call $\matr{SM}_k$.

  \smallskip
    Our first goal is to prove for even $p$ that $\norm{\matr{SM}_{k}}_{p, 1} \le 2 p$ and since
  $\norm{\vec{x}}_{p - 1} \ge \norm{\vec{x}}_p$ we can conclude that
  $\norm{\matr{SM}_{k}}_{p - 1, 1} \le \norm{\matr{SM}_{k}}_{p, 1} \le 2 p$. This implies 
  $\norm{\matr{SM}_{k}}_{p, 1} \le 2 (p + 1)$ for any $p$.

  \begin{claim}
    It holds that $\norm{\matr{SM}_{k}}_{p, 1} \le 2(p + 1)$ for any $p \in [1, \infty]$.
  \end{claim}
  \begin{proof}
      Using the Theorem \ref{thm:drakakisGeneral} and setting $r = p/(p - 1)$ we have that
    \[ \norm{\matr{SM}_{k}}_{p, 1} = \max_{\vec{z} \in \{-1, 1\}^k} \norm{\vec{z}^T \matr{SM}_k}_{p}. \]
    \noindent Now for every column $\vec{m}_i$ of $\matr{SM}_k$ we observe that the sum of the 
    coordinates is zero, that is $\sum_{j = 1}^k m_{j i} = 0$. Also all the element except the
    diagonal elements are non-positive and hence it is true that
    \[ \sum_{j = 1}^k \abs{m_{j i}} = 2 m_{i i}. \]
    \noindent But obviously $\abs{\vec{z}^T \vec{m}_i} \le \sum_{j = 1}^k \abs{m_{j i}}$ for all
    $\vec{z} \in \{-1, 1\}^k$. This implies that $\abs{\vec{z}^T \vec{m}_i} \le 2 m_{i i} = 2/i$.
    Therefore for any $\vec{z} \in \{-1, 1\}^k$ we have that
    \begin{equation} \label{eq:proof:lem:soft-maxMatricesNormBound1}
      \norm{\vec{z}^T \matr{SM}_k}_{r} = \left( \sum_{i = 1}^k \abs{\vec{z}^T \vec{m}_i}^r \right)^{1/r} \le 2 \left( \sum_{i = 1}^k \frac{1}{i^r} \right)^{1/r} \le 2 \left( \zeta(r) \right)^{1/r}
    \end{equation}
    \noindent where $\zeta(x)$ is the Riemann zeta function evaluated at $x$. Now we
    use the formula (2.1.16) of Chapter 2.1 of \cite{TitchmarshH1986} and we get
    that 
    \[ \zeta\left( \frac{p}{p - 1} \right) \le p. \]
    \noindent This implies that
    \[ \norm{\vec{z}^T \matr{SM}_k}_{r} \le 2 p^{(p - 1)/p} \le 2 p. \]
    \noindent This holds for any even $p$ since only in this case we can use Theorem
    \ref{thm:drakakisGeneral}, and this implies that for any $p$
    \[ \norm{\vec{z}^T \matr{SM}_k}_{r} \le 2 p^{(p - 1)/p} \le 2 (p + 1) \]
    \noindent as we argued in the beginning of the proof.
  \end{proof}
  
  Now it is obvious that $\norm{\cdot}_{p, q} \le \norm{\cdot}_{p, 1}$ and hence we have that
  $\norm{\matr{SM}_{(k, d)}}_{p, q} \le 2 (p + 1)$.
  
  Also, for any $p$ and any vector $\vecv \in \R^d$, we have 
  $\max_{\vecx: \norm{x}_p = 1} \abs{\vecv^{T} \vecx} = \norm{\vecv}_{q/(q - 1)}$. Applying this on
  rows of any matrix $A$, we get 
  \[ \norm{A}_{p, q} \le \left(\sum_i \norm{\veca_i}_{p/(p-1)}^q\right)^{1/q}. \]
  
  Therefore, for every $q > 1$, and using the formula (2.1.16) of Chapter 2.1 of
  \cite{TitchmarshH1986} and we get that 
  \[  \norm{\matr{SM}_{(k, d)}}_{p, q} \le \left(\sum_{i = 1}^k \frac{1}{i^q}\right)^{1/q} < \zeta(q)^{1/q} < \frac{q}{q - 1}. \]
  
  Finally we can use \eqref{eq:proof:lem:soft-maxMatricesNormBound1} and see that for any $q, p$
  \[ \norm{\vec{z}^T \matr{SM}_k}_{q} \le 2 \left( \sum_{i = 1}^k \frac{1}{i} \right) \le 2 \log k \]
  \noindent and this completes the proof of the lemma.
\end{proof}

\subsection{Proof of Theorem~\ref{thm:mainsoft-maxFunctionTV}}

  We first prove that $\vecf$ is continuous and that its output is always a probability
distribution over the $d$ coordinates, i.e. that its output belongs to $\Delta_{d - 1}$.
\smallskip

\paragr{Continuity of $\vecf$.} From the definition of $\vecf$ is easy to see that $\vecf$ is
piecewise linear, since it remains linear for all the regions where the order of the coordinates of
$\vecx$ is fixed and $k_{\vecx}$ is fixed. It is easy to see that the set of these regions is a
finite set and each region is a convex set. More formaly
\[ \mathcal{P}_{\vecf} = \left\{ \left\{ \vec{x} \mid \left( x_{\pi(1)} \ge x_{\pi(2)} \ge \cdots \ge x_{\pi(d)} \right) \wedge \left( x_{\pi(1)} - x_{\pi(k)} \le \delta \right) \right\} \mid \pi : [d] \to [d], k \in [d] \right\} \]
\noindent where $\pi$ has to be a permutation. Also the set of matrices that $\vec{f}$ uses is the
following
\[ \mathcal{A}_{\vecf} = \left\{ \frac{1}{\delta} \matr{P}^T \matr{SM}_{(k, d)} \matr{P} \mid k \in \nats, \matr{P} \text{ permutation matrix} \right\}. \]

  So its is clear that $\vecf$ is piecewise linear, but it is not clear that it should be 
continuous. To prove the continuity of $\vecf$ we will use the Lemmas 
\ref{lem:soft-maxMatricesContinuity1}, \ref{lem:soft-maxMatricesContinuity2}. Since $\vecf$ is
piecewise linear the only regions where $\vecf$ might not be continuous are the boundaries of the
regions $P_i \in \mathcal{P}_{\vecf}$. There are two types of such boundaries one because of the
change of the value $k_{\vecx}$ and because the ordering in $\vecx$ changes. First consider the
boundaries because of the change of $k_{\vecx}$ which for simplicity we call $k$ for the proof. At
the boundaries where $k$ decreases we have that $x_1 - x_k = \delta$ which implies 
$x_k = x_1 - \delta$. If we apply this in the definition of $\vecf$, then we get 
\begin{align*}
  \vec{f}(\vec{x}) & = \frac{1}{\delta} \matr{SM}_{(k, d)} \cdot \vec{x} + \vec{u}_k = \frac{1}{\delta} \left( \matr{SM}_{(k, d)} \left(\matr{I}_d + \matr{E}_{k, 1} - \matr{E}_{k, k} \right) \right) \vec{x} + \frac{1}{\delta} \delta \vec{m}^{(k, d)}_k + \vec{u}_k \\
                   & = \frac{1}{\delta} \matr{SM}_{(k - 1, d)} \cdot \vec{x} + \vec{m}^{(k, d)}_k + \vec{u}_k \\
                   & = \frac{1}{\delta} \matr{SM}_{(k - 1, d)} \cdot \vec{x} + \vec{u}_{k - 1}
\end{align*}
\noindent where at the second step we used Lemma \ref{lem:soft-maxMatricesContinuity1}. This implies
that at these boundaries the function remains continuous. The transition for $k$ to higher $k$ can
be proved exactly the same way. Now we consider the case where the ordering of $\vecx$ changes. In
this case we will have that for any two indices $i, j \in [d]$ that are changing order it is true
that $x_i = x_j$. But from \ref{lem:soft-maxMatricesContinuity2} and the definition of 
$\vecf(\vecx)$ we have that $f_i(\vecx) = f_j(\vecx)$. This implies that the relative order of $x_i$
and $x_j$ does not change the value of $\vec{f}$. Hence in the boundaries where the coordinates of
$\vecx$ change order $\vecf$ is continuous. Finally in any boundary that combines a change in $k$
and a change in the ordering of the coordinates of $\vec{x}$ we can combine the above arguments and
prove that $\vecf$ is continous at these boundaries too.
\smallskip

\paragr{Output of $\vecf$ in $\Delta_{d - 1}$.} We fix $k_{\vecx}$ to be $k$ and we consider without
loss of generality a vector $\vecx$ that satisfies 
\begin{align} \label{eq:orderedInput}
  x_1 \ge x_2 \ge \cdots \ge x_d.
\end{align}
\noindent Therefore
\[ \vec{f}(\vec{x}) = \frac{1}{\delta} \matr{SM}_{(k, d)} \cdot \vec{x} + \vec{u}_{k}. \]
\noindent From the definition of soft-max matrices we have that for any column $\vec{m}_j$ of
$\matr{SM}_{(k ,d)}$, $\sum_{i = 1}^d m_{i j} = 0$ and since $\sum_{i = 1}^d u_{k i} = 1$ we have
that for any $\vecx \in \R^d$, $\sum_{i = 1}^d f_i(\vec{x}) = 1$. Hence it remains to prove that
$f_i(\vec{x}) \ge 0$.

\smallskip
  Let $\vec{s}_i^T$ be the $i$th row of $\matr{SM}_{(k, d)}$. For $i > k$ we have 
$\vec{s}_i^T = \vec{0}^T$ and $u_{k i} = 0$, hence $f_i(\vec{x}) = 0$. On the other hand, if 
$i \le k$, we have that for
\[ f_i(\vec{x}) = \frac{1}{\delta} \sum_{j = 1}^d s_{i j} x_j + \frac{1}{k} = - \frac{1}{\delta k} x_1 + \frac{1}{\delta i} x_i + \frac{1}{\delta} \sum_{j = i + 1}^k s_{i j} x_j + \frac{1}{k} \]
\noindent but for $j > i$ $s_{i j} \le 0$ and because of \eqref{eq:orderedInput} we have that
\[ f_i(\vec{x}) \ge - \frac{1}{\delta k} x_1 + \frac{1}{\delta} \left( \frac{1}{i} + \sum_{j = i + 1}^k s_{i j} \right) x_2 + \frac{1}{k} = - \frac{1}{\delta k} x_1 + \frac{1}{\delta} \left( \sum_{j = i}^k s_{i j} \right) x_2 = - \frac{1}{\delta k}(x_1 - x_2) + \frac{1}{k} \]
\noindent now by the definition of $k$ we have that $-(x_1 - x_2) \ge -\delta$ and hence
\[ f_i(\vec{x}) \ge - \frac{1}{\delta k} \delta + \frac{1}{k} = 0. \]
\noindent This finishes the proof that $\vecf(\vecx)$ is a probability distribution.
\smallskip

  We are now ready to prove the two parts of Theorem
\ref{thm:mainsoft-maxFunctionTV}.

\paragr{Proof of 1.} Without loss of generality we can again assume that $\vecx$ satisfies
\eqref{eq:orderedInput} and we again fix $k = k_{\vecx}$. In this case the condition
$\norm{\vecx}_{\infty} - x_i > \delta$ translates to $i > k$. Then by the definition of $\vecf$ we
have that
\[ f_i(\vec{x}) = \vec{s}_i^T \vec{x} + u_{k i} \]
\noindent but by the definition of $\matr{SM}_{(k, d)}$ we have that $\vec{s}_i^T = \vec{0}^T$ and
$u_{k i} = 0$. These two imply $f_i(\vec{x}) = 0$.
\smallskip

\paragr{Proof of 2.} Since $\vec{f}$ is continuous and piecewise linear we can use Lemma
\ref{lem:LischitzTVFromNorms} and we get
\[ \norm{\vec{f}(\vec{x}) - \vec{f}(\vec{y})}_q \le \left( \max_{\matr{A} \in \mathcal{A}_{\vec{f}}} \norm{\matr{A}}_{p, q} \right) \cdot \norm{\vec{x} - \vec{y}}_{p} ~~~~ \forall \vec{x}, \vec{y} \in \R^{d}. \]
\noindent Now we have that the set $\mathcal{A}_{\vec{f}}$ is the following
\[ \mathcal{A}_{\vec{f}} = \left\{ \frac{1}{\delta} \matr{P}^T \matr{SM}_{(k, d)} \matr{P} \mid k \in \nats, \matr{P} \text{ permutation matrix} \right\} \]
\noindent and since $\matr{P}$ is a permutation matrix we have that
\[ \norm{\matr{P}^T \matr{SM}_{(k, d)} \matr{P}}_{p, q} = \norm{\matr{SM}_{(k, d)}}_{p, q} \]
\noindent which implies
\[ \norm{\vec{f}(\vec{x}) - \vec{f}(\vec{y})}_q \le \frac{1}{\delta} \left( \max_{k \in [d]} \norm{\matr{SM}_{(k, d)}}_{p, q} \right) \cdot \norm{\vec{x} - \vec{y}}_{p} ~~~~ \forall \vec{x}, \vec{y} \in \R^{d}. \]
\noindent Finally using Lemma \ref{lem:sofmaxMatricesNormBound} we have that
\[ \norm{\vec{f}(\vec{x}) - \vec{f}(\vec{y})}_q \le \frac{2 \min\set{\frac{q}{q- 1}, p + 1, \log d}}{\delta} \norm{\vec{x} - \vec{y}}_{p} ~~~~ \forall \vec{x}, \vec{y} \in \R^{d}. \]
\noindent This completes the proof of the theorem.
\section{Proofs of Lower Bounds in Section~\ref{sec:plsm_LB}} 
\label{app:plsm_LB}

  In this section we provide the proofs of Theorem 
\ref{thm:totalVariationLowerBound} and Theorem \ref{thm:exponentialLowerBound}.

\subsection{Proof of Theorem \ref{thm:totalVariationLowerBound}} \label{app:totalVariationLowerBound}

We will show our proof of all the dimensions $d$ of the form $d = 2^{2^k}$, $k \in \nats_+$. Then we can deduce that
asymptotically our lower bound holds. We use an induction argument with base case $d = 2$ and inductive step from $d$ to
$d^2$. All the logarithms in this proof are with base $2$.

For technical reasons we prove a compact-domain version of the theorem restricted to $\R_+^d$. This leads to a stronger lower bound and implies the theorem statement on $\mathbb R^d$.
\paragraph{Setup.}
Fix $\delta > 0$, $k \in \N$. For $d = 2^{2^k} \ge 2$, consider soft-max functions
\[
\vec f : [0, M_{d, \delta}]^d \to \Delta_{d-1}.
\]
where we set $M_{d, \delta} = 2^{k + 2} \delta = 4 \log(d) \delta$. Since $\vecf$ is defined over $\R_+^d$, $\vec f$ is \emph{$\delta$-approximate} if
\[
\forall \vec x\in[0,M_{d, \delta}]^d:\qquad
\langle \vec x, \vec f(\vec x)\rangle \ge \|\vec x\|_\infty - \delta.
\]
We measure smoothness using the $(\ell_\infty,\ell_1)$-Lipschitz constant.

\medskip
\paragr{Induction Base, $\boldsymbol{d = 2}$.} In this case we have that $\vec{f}(\vec{x}) = (f_1(x_1, x_2), 1 - f_1(x_1, x_2))$, with $x_1, x_2 \in [0, 4 \delta]$, and for simplicity we use
the notation $f$ to refer to $f_1$. We will prove that the $\ell_{\infty}$ to $\ell_1$ Lipschitz constant of $\vec{f}$ is at
least $1/(8 \delta)$ even in the restricted subregion where $x_1 + x_2 = 4 \delta$. In this region the problem
becomes single dimensional since $\vec{f}(\vec{x}) = (f_1(x_1, 4 \delta - x_1), 1 - f_1(x_1, 4 \delta - x_1))$ and the only freedom of
$\vec{f}$ is to decide the single dimensional function $f(x) = f_1(x, 4 \delta - x)$. The approximation constraint implies that
\[ \max\{x, 4\delta - x\} - x f(x) - (4\delta - x)(1 - f(x)) \le \delta \Leftrightarrow (4\delta - 2 x) f(x) \le \delta - \max\{x, 4\delta - x\} + 4 \delta - x. \]
\noindent This implies the following
\[ f(0) \le \frac{1}{4} \quad \text{and} \quad f(4 \delta) \ge \frac{3}{4}.\]
Which implies that
\[ \frac{|f(4 \delta) - f(0)|}{4 \delta} \ge \frac{1}{8 \delta} \]
Hence, the Lipschitz constant of $f$ is at least $1 / (8 \delta)$. But it is easy to that
\[ \norm{\vecf(x, 4\delta - x) - \vecf(x', 4\delta - x')}_1 = 2 |f(x) - f(x')| \]
which implies that the $(\ell_{\infty}, \ell_1)$-Lipschitz constant of $\vecf$ is lower bounded by twice the Lipschitz constant of $f$. Therefore, the $(\ell_{\infty}, \ell_1)$-Lipschitz constant of $\vecf$ is at least $1/(4\delta)$ and the base case follows.

\medskip
\paragr{Inductive Step, from $\boldsymbol{d}$ to $\boldsymbol{d^2}$.} We assume by inductive
hypothesis that for any soft maximum function $\vecf : [0, M_{d, \delta}]^d \to \Delta_{d - 1}$ in $d$ dimensions, with Lipschitz constant
at most $\log(d) /(4 \delta)$ has expected approximation loss at least $\delta$. We then prove
that for any soft maximum function $\vecf$ in $d^2$ dimensions with Lipschitz constant at most
$\log(d) /(4 \delta)$ has expected approximation loss at least $2 \cdot \delta$. This in turn
implies that if $\vecf$ has Lipschitz constant at most $2 \log(d) /(4 \delta)$ then $\vecf$ has approximation loss at least $\delta$.

\noindent We restrict our attention to a subspace of $[0, M_{d^2, \delta}]^{d^2}$ that is generated by
$\left([0, M_{d, \delta}]^d\right)^2$ by the following map 
$\vec{g} : \left([0, M_{d, \delta}]^d\right)^2 \to [0, M_{d^2, \delta}]^{d^2}$ defined as
\[ g_\ell(\vec{x}, \vec{y}) = x_{\ell \bmod d} + y_{\ell \div d}. \]

\noindent For $d = 2^{2^k}$, we have $M_{d^2, \delta} = 2^{k+3} \delta = 2 M_{d, \delta}$ from the definition of $M_{d, \delta}$, it is easy to see that $\vecg$ indeed maps $\left([0, M_{d, \delta}]^d\right)^2$ to $[0, M_{d^2, \delta}]^{d^2}$. On these instances of $[0, M_{d^2, \delta}]^{d^2}$ we want to view the space of alternatives $[d^2]$ as a tensor space
$[d] \otimes [d]$ and this is the role of the mapping $\vec{g}$. We also want to view the output distribution as a
product distribution over $[d] \otimes [d]$ but since we cannot force this independence condition in the outputs of $\vec{f}$, we can only define the marginal distributions
of $\vec{f}(\vec{z})$ to the coordinates $\ell$ that have index with the same value $\ell \bmod d$, and the coordinates $\ell$
that have the same value $\ell \div d$. We define $\vec{q} : \left([0, M_{d, \delta}]^{d}\right)^2 \to \Delta_d$ to be the marginal distribution to the
coordinates $\ell$ that have index with the same value $\ell \div d$ and $\vec{r} : \left([0,M_{d, \delta}]^{d}\right)^2 \to \Delta_d$ to be  the marginal
distribution to the coordinates $\ell$ that have the same value $\ell \bmod d$. Formally, we have that
\begin{align*}
               & q_i(\vecx,\vecy) = \sum_{j = 1}^d f_{i d + j}(\vecg(\vecx,\vecy)) \\
  \text{ and } & r_j(\vecx,\vecy) = \sum_{i = 1}^d f_{i d + j}(\vecg(\vecx,\vecy)).
\end{align*}
\noindent Now, due to the definitions of $\vec{g}$, $\vec{q}$, and $\vec{r}$ we can verify using simple algebraic calculations that 
\begin{align*}
               \norm{\vec{g}(\vec{x}, \vec{y})}_{\infty} & = \norm{\vec{x}}_{\infty} + \norm{\vec{y}}_{\infty} \quad \text{for all } \quad \vecx, \vecy \in [0, M_{d,\delta}]^d\\
  \text{ and } ~~~ \langle \vec{f}(\vec{g}(\vec{x}, \vec{y})), \vec{g}(\vec{x}, \vec{y}) \rangle & = \langle \vec{q}(\vec{x}, \vec{y}), \vec{x} \rangle + \langle \vec{r}(\vec{x}, \vec{y}), \vec{y} \rangle.
\end{align*}
\noindent Hence,
\begin{equation}
\norm{\vec{g}(\vec{x}, \vec{y})}_{\infty} - \langle \vec{f}(\vec{g}(\vec{x}, \vec{y})), \vec{g}(\vec{x}, \vec{y}) \rangle = \underbrace{\norm{\vec{x}}_{\infty} - \langle \vec{q}(\vec{x}, \vec{y}), \vec{x} \rangle}_{\delta_1(\vec{x}, \vec{y})} + \underbrace{\norm{\vec{y}}_{\infty} - \langle \vec{r}(\vec{x}, \vec{y}), \vec{y} \rangle}_{\delta_2(\vec{x}, \vec{y})} \label{eq:fixGlobalError}
\end{equation}

\noindent We now define a continuous two-player game with the following players:
\begin{enumerate}
  \item the first player picks a strategy $\vecx \in [0, M_{d, \delta}]^d$ and has reward
        equal to $\delta_1(\vecx, \vecy)$, and
  \item the second player picks a strategy $\vecy \in [0, M_{d, \delta}]^d$ and has reward
        equal to $\delta_2(\vecx, \vecy)$.
\end{enumerate}

\noindent It is easy to see that since $\vecf$ and $\vecg$ are continuous, both $\vecq$ and $\vecr$ are
continuous because finite sums preserve continuity, which implies that $\delta_1$ and $\delta_2$ are continuous. It is well known 
then from the theory of continuous games that there exists a mixed Nash Equilibrium in the 
game that we described above by invoking Glicksberg's theorem \cite{Glicksberg1952}. This means that there exists a pair of distributions 
$\calD_x$, $\calD_y$ with support $[0,M_{d, \delta}]^d$ such that
\begin{enumerate}
  \item if $\vecx^{\star}$ in the support of $\calD_x$ then 
        $\vecx^{\star} = \argmax_{\vecx \in [0,M_{d, \delta}]^d} \Exp_{\vecy \sim \calD_y}\b{\delta_1(\vecx, \vecy)}$, and
  \item if $\vecy^{\star}$ in the support of $\calD_y$ then
        $\vecy^{\star} = \argmax_{\vecy \in [0,M_{d, \delta}]^d} \Exp_{\vecx \sim \calD_x}\b{\delta_2(\vecx, \vecy)}$.
\end{enumerate}

\noindent Let us know define the following functions
\begin{itemize}
  \item $\bar{\vecq}(\vecx) = \Exp_{\vecy \sim \calD_y}\b{\vecq(\vecx, \vecy)}$,
  \item $\bar{\vecr}(\vecy) = \Exp_{\vecx \sim \calD_x}\b{\vecr(\vecx, \vecy)}$,
  \item $\bar{\delta}_1(\vecx) = \Exp_{\vecy \sim \calD_y}\b{\delta_1(\vecx, \vecy)} = \norm{\vecx}_{\infty} - \langle \bar{\vecq}(\vecx), \vecx \rangle$, and
  \item $\bar{\delta}_2(\vecy) = \Exp_{\vecx \sim \calD_x}\b{\delta_2(\vecx, \vecy)} = \norm{\vecy}_{\infty} - \langle \bar{\vecr}(\vecy), \vecy \rangle$
\end{itemize}
\noindent where in the definition of the last two functions we have used the linearity of
expectation. From the Nash Equilibrium conditions that define $\calD_x$ and $\calD_y$ we have that 
\[ \Exp_{\vecx \sim \calD_x, \vecy \sim \calD_y}\b{\delta_1(\vecx, \vecy) + \delta_2(\vecx, \vecy)} = \max_{\vecx \in [0,M_{d, \delta}]^d}\set{\bar{\delta}_1(\vecx)} + \max_{\vecy \sim [0,M_{d, \delta}]^d}\set{\bar{\delta_2(\vecy)}} \]
which in turn implies the following
\begin{equation} \label{eq:app:mainLowerBound:errorBound:1}
  \max_{\vecx, \vecy \in [0,M_{d, \delta}]^d}\set{\delta_1(\vecx, \vecy) + \delta_2(\vecx, \vecy)} \ge
  \max_{\vecx \in [0,M_{d, \delta}]^d}\set{\bar{\delta}_1(\vecx)} + \max_{\vecy \sim [0,M_{d, \delta}]^d}\set{\bar{\delta_2(\vecy)}}.
\end{equation}

\noindent Our next goal is to relate the Lipschitzness of $\vecf$ with the Lipschitzness of $\bar{\vecq}$
and $\bar{\vecr}$. Observe that
\begin{align}
  \norm{\vec{g}(\vec{x}, \vec{y}) - \vec{g}(\vec{x'}, \vec{y'})}_{\infty} & \le \norm{\vec{x} - \vec{x'}}_{\infty} + \norm{\vec{y} - \vec{y'}}_{\infty} \label{eq:app:mainLowerBound:LipschitznessBound:domain} \\
  \norm{\vec{q}(\vec{x}, \vec{y}) - \vec{q}(\vec{x'}, \vec{y'})}_1 & \le \norm{\vec{f}(\vec{g}(\vec{x}, \vec{y})) - \vec{f}(\vec{g}(\vec{x'}, \vec{y'}))}_1 \label{eq:app:mainLowerBound:LipschitznessBound:q:1} \\
  \norm{\vec{r}(\vec{x}, \vec{y}) - \vec{r}(\vec{x'}, \vec{y'})}_1 & \le \norm{\vec{f}(\vec{g}(\vec{x}, \vec{y})) - \vec{f}(\vec{g}(\vec{x'}, \vec{y'}))}_1 \label{eq:app:mainLowerBound:LipschitznessBound:r:1}
\end{align}
\noindent where the first inequality follows from simple calculations and the second and third
inequality follow from the well-known fact that the total variation distance of a distribution is lower
bounded by the total variation of its marginals.

  Now we remind that we have assumed that $\vecf$ has $(\ell_{\infty}, \ell_1)$-Lipschitz constant
that is at most $L = \log(d) / (4 \delta)$. Using the fact that the $\ell_1$ norm is a convex
function and using the Jensen inequality we have that
\begin{align}
  \norm{\bar{\vec{q}}(\vecx) - \bar{\vec{q}}(\vecx')}_1 & \le \Exp_{\vecy \sim \calD_y}\b{\norm{\vec{q}(\vecx, \vecy) - \vec{q}(\vecx', \vecy)}_1} \nonumber \\
  & \le \Exp_{\vecy \sim \calD_y}\b{\norm{\vec{f}(\vecg(\vecx, \vecy)) - \vec{f}(\vecg(\vecx', \vecy))}_1} \le L \cdot \norm{\vecx - \vecx'}_{\infty} \label{eq:app:mainLowerBound:LipschitznessBound:q:2}
\end{align}
\noindent where the first inequality is due to Jensen, the second inequality follows from
\eqref{eq:app:mainLowerBound:LipschitznessBound:q:1} and the last inequality follows from the
$(\ell_{\infty}, \ell_1)$-Lipschitz constant of $\vecf$ and 
\eqref{eq:app:mainLowerBound:LipschitznessBound:domain}. Similarly,
\begin{align}
  \norm{\bar{\vec{r}}(\vecy) - \bar{\vec{r}}(\vecy')}_1 & \le L \cdot \norm{\vecy - \vecy'}_{\infty} \label{eq:app:mainLowerBound:LipschitznessBound:r:2}.
\end{align}

  It hence follows that both $\bar{\vecq}$ and $\bar{\vecr}$ are soft-max functions in $d$ 
dimensions with Lipschitz constant at most $L = \log(d) / (4 \delta)$. Hence, from our inductive 
hypothesis we have that the approximation error of both $\bar{\vecq}$, $\bar{\vecr}$ is at least
$\delta$, or more formally
\[ \max_{\vecx \in [0,M_{d, \delta}]^d} \bar{\delta_1}(\vecx) \ge \delta ~~~~~~ \text{and} ~~~~~~ \max_{\vecy \in [0,M_{d, \delta}]^d} \bar{\delta}_2(\vecy) \ge \delta. \]
\noindent Putting the above inequalities together with 
\eqref{eq:app:mainLowerBound:errorBound:1} we get that
\[\max_{\vecx, \vecy \in [0,M_{d, \delta}]^d} \set{\delta_1(\vecx, \vecy) + \delta_2(\vecx, \vecy)} \ge 2 \delta.\] 
Now invoking \eqref{eq:fixGlobalError} we get that
\[\max_{\vecx, \vecy \in [0,M_{d, \delta}]^d} \norm{\vec{g}(\vec{x}, \vec{y})}_{\infty} - \langle \vec{f}(\vec{g}(\vec{x}, \vec{y})), \vec{g}(\vec{x}, \vec{y}) \rangle \ge 2 \delta\]
and obviously
\[\max_{\vecz \in [0,M_{d^2, \delta}]^{d^2}} \norm{z}_{\infty} - \langle \vecf(\vecz), \vecz \rangle \ge \max_{\vecx,\vecy \in [0,M_{d, \delta}]^d} \norm{\vec{g}(\vec{x}, \vec{y})}_{\infty} - \langle \vec{f}(\vec{g}(\vec{x}, \vec{y})), \vec{g}(\vec{x}, \vec{y}) \rangle \ge 2 \delta.\]
Hence, we proved our inductive step which concludes the proof of our theorem.

\subsection{Proof of Theorem~\ref{thm:exponentialLowerBound}}

  We set $\vec{x} = (x, 0, \dots, 0)^T$ and $\vec{y} = (y, 0, \dots, 0)^T$, with $y > x$. Then we
have
\[ \Expon(\vec{x}) = \left( \frac{e^{\alpha x}}{e^{\alpha x} + (d - 1)}, \frac{1}{e^{\alpha x} + (d - 1)}, \cdots, \frac{1}{e^{\alpha x} + (d - 1)} \right)^T \]
\[ \Expon(\vec{y}) = \left( \frac{e^{\alpha y}}{e^{\alpha y} + (d - 1)}, \frac{1}{e^{\alpha y} + (d - 1)}, \cdots, \frac{1}{e^{\alpha y} + (d - 1)} \right)^T. \]
\noindent Since $y > x$, we compute
\begin{align*}
  \norm{\Expon(\vec{x}) - \Expon(\vec{y})}_1 & = \left( \frac{e^{\alpha y}}{e^{\alpha y} + (d - 1)} - \frac{e^{\alpha x}}{e^{\alpha x} + (d - 1)} \right) \\
  & - (d - 1) \left( \frac{1}{e^{\alpha y} + (d - 1)} - \frac{1}{e^{\alpha x} + (d - 1)} \right)
\end{align*}
\noindent and $\norm{\vec{x} - \vec{y}}_p = y - x$. Now let
\[ h(z) = \frac{e^{\alpha z}}{e^{\alpha z} + (d - 1)} - (d - 1) \frac{1}{e^{\alpha z} + (d - 1)} = \frac{e^{\alpha z} - (d - 1)}{e^{\alpha z} + (d - 1)}  \]
\noindent our goal to maximize, with respect to $x, y \in \realsp$ with $y \ge x$, the ratio
\[ \frac{\norm{\Expon(\vec{x}) - \Expon(\vec{y})}_1}{\norm{\vec{x} - \vec{y}}_p} = \frac{h(y) - h(x)}{y - x}. \]
\noindent Because of the mean value theorem this is equivalent with maximum with respect to $z \in \realsp$ the derivative of
$h$, $h'(z)$. But we have
\[ h'(z) = \frac{\alpha e^{\alpha z} \left( e^{\alpha z} + (d - 1) \right) - \alpha e^{\alpha z} \left( e^{\alpha z} - (d - 1) \right)}{\left( e^{\alpha z} + (d - 1) \right)^2} = 2 \alpha \frac{e^{\alpha z} (d - 1)}{\left( e^{\alpha z} + (d - 1) \right)^2}. \]
\noindent Now we set $z = \frac{\log d}{\alpha}$ and we get for $d \ge 2$
\[ h'\left( \frac{\log d}{\alpha} \right) = 2 \alpha \frac{d (d - 1)}{(2 d - 1)^2} \ge \frac{\alpha}{4}. \]

  Finally since the absolute approximation error of the exponential mechanism with parameter $\alpha$ is $\log d / \alpha$, to
get $\delta$ absolute error we have to set $\alpha = \log d / \delta$ and hence for this regime
\[ c \ge \frac{\log d}{4 \delta} \]
\noindent and the proof of the theorem is completed.

\section{Application to Mechanism Design} \label{sec:singleItem}

  In this section we show how to design a digital auction with limited supply and worst case
guarantees. As we will see to do so we need to relax the incentive compatibility constraints to
approximate incentive compatibility in the framework as in \cite{McSherryT07}. In this setting we
fix an anonymous price for all the agents regardless of whether their values follow the same
distribution of not. In this case we show that we can extract almost the optimal revenue among all
the fixed price auctions.

  Compared to the results of \cite{McSherryT07} and \cite{BalcanBHM05} our mechanism can
interpolate between both of the results. Most importantly our results, in contrast to both 
\cite{McSherryT07} and \cite{BalcanBHM05} achieves a worst case guarantee instead of a guarantee
in expectation or with high probability. Another improvement of our result is that it holds even
if we do not assume unlimited supply but we only have finite supply of the item to sell.

  We start with the next Section \ref{sec:auctionModel} with the basic definitions and formulation
of the mechanism and auction design problem.

\subsection{Definitions and Preliminaries} \label{sec:auctionModel}

  We first give some necessary basic definitions of design auctions for selling $k$ identical
items to $n$ independent bidders with unit demand valuations.
\smallskip

\paragr{Items.} We have $k$ identical items for sell.

\paragr{Bidders.} We have $n$ independent bidders with unit demand valuations
over the $k$ item to sell. The bidders are clustered in $t$ classes and let
$t(i)$ be the class of bidder $i$. The value of bidder $i \in [n]$ for any of
the items is $v_i \in [0, H]$ where $H$ is the maximum possible value that we
assume to be known. We also assume that $v_i$ it is drawn from a distribution
$\Distr_{t(i)}$. We assume that all the random variables $v_i$ are
independent from each other.
\smallskip

\paragr{Mechanism.} A mechanism $M$ is a function
$M : \realsp^{n} \to \Delta_n^k \times \realsp^n$ that takes as input the bid
of the players and outputs $k$ probability distributions
$\matr{A} = (\vec{a}_1, \dots, \vec{a}_k) \in \Delta_n^k$ over the bidders
that determines the probability that each bidder is going to receive the item
$j$, together with a non-negative value $p_i$ for every bidder $i$ that
determines the money bidder $i$ will pay. We write
$M(\vec{v}) = (\matr{A}, \vec{p})$ and we call $\vec{A} \in \Delta_n^k$ the
allocation rule of the mechanism $M$ and $\vec{p}$ the payment rule of $M$.
\smallskip

\paragr{Bidders Utility.} We assume that the bidders are unit-demand and they
have quasi-linear utility, i.e. that the utility function
$u_i : \Delta_n \times \realsp^n \to \reals$ of each bidder is equal to
$u_i(\vec{A}, \vec{p}) = \max_{j} ( a_{ij} v_i ) - p_i$.

\paragr{Revenue Objective.} For every mechanism $M$ the revenue
$\Rev{M, \vec{v}}$ the designer gets in input $\vec{v}$ is equal to
$\Rev{M, \vec{v}} = \sum_{i \in [n]} p_i$ where $\vec{p}$ is the vector of
prices that the mechanism $M$ assigns to the agents in input $\vec{v}$. By
$\Rev{M}$ we denote the expected value of the mechanism $M$ when the values
$\vec{v}$ are drawn from their distributions, i.e.
$\Rev{M} = \Exp\left[\Rev{M, \vec{v}}\right]$.
\smallskip

\paragr{Incentive Compatibility.} A mechanism $M$ is called \textit{dominant strategy incentive compatible} (DSIC) or simply
\textit{incentive compatible} (IC) if the bidders cannot increase their revenue by misreporting their bids. More precisely we say
that $M$ satisfies incentive compatibility if for every bidder $i$
\begin{align} \label{eq:truthfulnessDefinition}
u_i(M(v_i, \vec{v}_{-i})) \ge u_i(M(v'_{i}, \vec{v}_{-i})) & ~~~~~~~~ \forall ~ v_i, v'_i, \vec{v}_{-i}.
\end{align}
\noindent Also we say that
$M$ is \textit{$\eps$-incentive compatible if for every bidder $i$}
\begin{align} \label{eq:epsTruthfulnessDefinition}
u_i(M(v_i, \vec{v}_{-i})) \ge \cdot u_i(M(v'_{i}, \vec{v}_{-i})) - \eps & ~~~~~~~~ \forall ~ v_i, v'_i, \vec{v}_{-i}.
\end{align}

\paragr{Individual Rationality.} We say that a mechanism $M$ satisfies
individual rationality if for every bidder $i$
$u_i(M(\vec{v})) \ge 0$ for all $\vec{v} \in \realsp^n$.
\smallskip

\paragr{Optimal Revenue over a Ground Set.} Let
$\Ground = \{M_1, \dots, M_d\}$ be a set of mechanisms which we call
\textit{ground set}, we define the maximum revenue of $\Ground$ at input
$\vec{v}$ as
$\OptM{\Ground, \vec{v}} = \max_{M \in \Ground} \Rev{M, \vec{v}}$. Also we
define maximum expected revenue achievable by any
mechanism in $\Ground$ to be
$\OptM{\Ground} = \max_{M \in \Ground} \Rev{M}$.
\smallskip

\noindent The mechanisms that we describe in this section involve a smooth
selection of a mechanism among the mechanisms in a carefully chosen ground
set of incentive compatible and individual rational mechanisms $\Ground$.
\smallskip

\paragr{Soft Maximizer Mechanism.} Let $\Ground = \{M_1, \dots, M_d\}$ be a
ground set of incentive compatible and individually rational mechanism. We
define the mechanism $Q[\Ground, \vec{f}]$ to be the mechanism that chooses
one of the mechanisms in $[d]$ randomly from the probability distribution
that output the soft maximum function $\vec{f}$ with input the vector
$\vec{x} = (\Rev{M_1, \vec{v}}, \dots, \Rev{M_d, \vec{v}})$.
\smallskip

\noindent The following lemma proves the incentive compatibility properties
of the mechanism $Q[\Ground, \vec{f}]$ when the $\vec{f}$ satisfies some
stability properties. For a proof of this lemma we refer to the proof of
Lemma 3 in McSherry and Talwar \cite{McSherryT07}.

\begin{lemma} \label{lem:lipschitzToIC}
    Let the bidders valuations come from the interval $[0, H]$, let also
  $\Ground = \{M_1, \dots, M_d\}$ be a ground set of incentive compatible
  and individually rational mechanism and $\vec{f}$ be a soft maximum function
  that is $(\ell_p, \ell_1)$-Lipschitz with Lipschitz constant $L = \eps/S_{\chi}(\textsc{Rev})$.
  Then the mechanism $Q[\Ground, \vecf]$ is individually rational and $\eps$-incentive compatible.
\end{lemma}

\subsection{Selling Digital Goods with Anonymous Price} \label{sec:anonymous}

  The single parameter auctions are arguably the most classical setting in the
mechanism design literature. Myerson, in his seminal work \cite{Myerson1981},
proved that among all the possible auction designs the revenue is maximized
by a second price auction with reserve price. The basic assumptions of his
framework though is the assumption that the auctioneer has a prior belief for
the values of the different bidders and she tries to maximize her expected
revenue in this Bayesian setting. This assumption is the major milestone in
applying the Myerson's auction in practice. Trying to relax this assumption, a
line of theoretical computer science work studied the maximization of revenue
when we only have access to samples that come from the bidders distribution
and not access to the entire distribution
\cite{RoughgardenTY12, DhangwatnotaiRY15, ColeR14, MorgensternR15, DevanurHP16, CaiD17}.
Although these works make a very good progress on understanding the optimal
auctions and make them more practical there are still some drawbacks that
make these auctions not applicable in practice.

\begin{enumerate}
  \item \textbf{Buyers may strategize in the collection of samples.} If the
  buyers know that the seller is going to collect samples to estimate the
  optimal auction to run then they have incentives to strategize so that the
  seller chooses lower prices and hence they get more utility.

  \item \textbf{Constant approximation is not always a satisfying guarantee.}
  The constant approximation is a worst case guarantee and hence the constant
  approximation mechanisms might fail to get almost optimal revenue even in
  the instances where this is easy. A popular alternative in practical
  applications of mechanism design is to choose the optimal from a set of
  simple mechanisms.
\end{enumerate}

\noindent Because of these reasons, 1. and 2., the implementation and the
theoretical guarantees of the mechanism $Q[\Ground, \vec{f}]$ becomes a
relevant problem. The ground set of mechanisms that we consider in this
section is a subset of the second price selling separately auctions with a
single reserved price, which we call set of anonymous auctions and we denote
by $\Ground_A$. We are now ready to prove the main result of this section.

\begin{theorem} \label{thm:singleItemPriorFree}
    Consider a $k$ identical item auction instance with unit demand bidder's
  and values in the range $[0, H]$. Then there exists a ground set of
  mechanisms $\hat{\Ground} \subseteq \Ground_A$ such that for all
  $\vec{v} \in [0, H]^n$ and for any of the possible outputs of
  $Q\left[\hat{\Ground}, \plsm^{\eta}\right]$ with input $\vec{v}$ it holds
  that
  \[ \Rev{Q\left[\hat{\Ground}, \plsm^{\eta}\right], \vec{v}} \ge (1 - \delta) \OptM{\Ground_A, \vec{v}} - 4 \left( \frac{1}{\delta} - 1\right) \frac{H}{\eps} \]
  \noindent where $\plsm^{\eta}$ is the soft maximum function defined in
  \eqref{eq:soft-maxFunctionTV} with parameter such that $\plsm$ is
  $\eps$-Lipschitz in Total Variation Distance. Moreover
  $Q[\hat{\Ground}, \plsm]$ is individually rational and
  $\eps \cdot H$-incentive compatible.
\end{theorem}

\begin{proof}
  Let $[0, H]$ be the range of prices for the single item auction. We fix a
positive real number $\delta$ and we use the discretization $\mathcal{P}$ of
$[0, H]$, where $\mathcal{P} = \{p_1, \dots, p_d\}$ and
$p_i = H \cdot (1 - \delta)^i$. Let also $\alpha = p_d$. We are now ready to
define the ground set of mechanisms $\hat{\Ground} = \{M_1, \dots, M_d\}$
where $M_i$ is the second price auction with reserved price equal to $p_i$.
The size of $\hat{\Ground}$ is
\[ d = \log \left(\frac{\alpha}{H}\right)/\log(1 - \delta) \le 2 \log \left(\frac{H}{\alpha}\right)/\delta \]
\noindent where the last inequality follows assuming that $\delta \le 1/2$.
As we described, we will run our mechanism $\plsm$, with objective function
$\Revf$. In order to be able to apply our main theorem about the $\plsm$
mechanism we will bound the $\ell_1$-sensitivity of the vector
$\vec{x} = (\Rev{M_1, \vec{v}}, \dots, \Rev{M_d, \vec{v}})$ with respect the
change of the bid of one agent. Hence we need to bound the quantity
\[ \sum_{i = 1}^d \abs{\Rev{M_i, (v_i, \vec{v}_{-i})} - \Rev{M_i, (v'_i, \vec{v}_{-i})}} \le (1 - \delta) \frac{H}{\delta}. \]
This inequality holds because for every agent $i$ the total change that agent
$i$ can make in the revenue objective of all the alternatives is at most
\[ \sum_{i = 1}^{d} (1 - \delta)^i H \le \left(\frac{1}{\delta} - 1\right) H, \]
\noindent which implies that for our setting
$S_1(\Revf) \le \left(\frac{1}{\delta} - 1\right) H$.

  The approximation loss of our mechanism has three components: (1) we loose
$\delta \OPT$ because of the discretization of the price of every item, (2)
we loose $\alpha$ from every item because we need the ground set to be finite
and (3) we loose $\eta$ because we use the soft maximization algorithm
$\plsm^{\eta}$. For the last part and since we need $\plsm^{\eta}$ to be
$\eps$-Lipschitz in total variation distance we have that
\[ \eps = \frac{4}{\eta} S_1(\Revf) \le \frac{4}{\eta} \left(\frac{1}{\delta} - 1 \right) H \implies \eta \le \frac{4}{\eps}\left(\frac{1}{\delta} - 1 \right) H. \]
\noindent Finally applying Theorem \ref{thm:mainsoft-maxFunctionTV} the
theorem follows.
\end{proof}

  If we assume that $H = O(1)$ then by setting
$\delta \leftarrow \frac{1}{\sqrt{\OPT}}$ and $\eps \leftarrow \eps \cdot H$
we recover the result of \cite{BalcanBHM05}, with relaxed incentive
compatibility, but even in the case of limited supply and having a worst case
guarantee.

\begin{corollary} \label{cor:digital1}
    Consider a $k$ identical item auction instance with unit demand bidder's
  and values in the range $[0, H]$. If we fix $H$ then there exists a
  mechanism $M$ such that for any $\vec{v} \in [0, H]^n$, for all
  $\vec{v} \in [0, H]^n$ and for any of the possible outputs of
  $M$ with input $\vec{v}$ it holds that
  \[ \Rev{M, \vec{v}} \ge \OptM{\Ground_A} - O\left(\frac{1}{\eps} \sqrt{\OptM{\Ground_A}} \right) \]
  \noindent where $M$ is individually rational and $\eps$-incentive
  compatible.
\end{corollary}

  Another corollary can be directly derived by applying a discretized version
of the Theorem 9 of \cite{McSherryT07} but replacing the exponential mechanism
with the $\plsm$ mechanism. Then as we explained in Section \ref{sec:plsm}
the guarantees will hold in the worst case and not in expectation.

\begin{corollary} \label{cor:digital2}
    Consider a $k$ identical item auction instance with unit demand bidder's
  and values in the range $[0, H]$. If we fix $H$ then there exists a
  mechanism $M$ such that for any $\vec{v} \in [0, H]^n$, for all
  $\vec{v} \in [0, H]^n$ and for any of the possible outputs of
  $M$ with input $\vec{v}$ it holds that
  \[ \Rev{M, \vec{v}} \ge \OptM{\Ground_A} - O\left(\frac{1}{\eps} \log \left(\OptM{\Ground_A} \cdot k\right) \right) \]
  \noindent where $M$ is individually rational and $\eps$-incentive
  compatible.
\end{corollary}

  As we can see Corollary \ref{cor:digital1} and Corollary \ref{cor:digital2}
are not directly comparable since in Corollary \ref{cor:digital2} the
$\log(k)$ factor in the approximation error appears that misses from Corollary
\ref{cor:digital1}.


\section{Maximization of Submodular Functions} \label{sec:submodular}

  In this section we consider the problem of differential privately maximizing a submodular
function, under cardinality constraints. For this problem we apply the power mechanism and we
compare our results with the state of the art work of Mitrovic et al. \cite{MitrovicBKK17}. 
We observe that when the input data set is only $O(1)$-multiplicative insensitive power mechanism
has an error that is asymptotically smaller than the corresponding error from the state of the art
algorithm of  Mitrovic et al. \cite{MitrovicBKK17}. This result is formally stated in Corollary 
\ref{cor:monotoneSubmodularCardinalityWithAssumption}.



\medskip

\noindent As discussed in Section \ref{sec:multiplicative:submodular}, to solve the submodular 
maximization under cardinality constraints we use the Algorithm 1 of \cite{MitrovicBKK17}, where
we replace the exponential mechanism in the soft maximization step with the power mechanism.

\clearpage
\noindent \textbf{Algorithm 1} (Algorithm 1 of \cite{MitrovicBKK17}):

\noindent \textbf{Input:} submodular function $h$, soft maximization function $\vecg$, $k \in \N$.

\noindent \textbf{Output:} $S \subseteq \Domain$ such that $\abs{S} = k$.

\begin{Enumerate}
  \item Initialize $S_o = \emptyset$. Let $\abs{\Domain} = d$ and $\Domain = \{v_1, \dots, v_d\}$.
  \item For $i \in [k]$:
        \begin{Enumerate}[label=\alph*.]
          \item Define $q_i : \Domain \setminus S_{i - 1} \to \reals$ as
                \[ q_i(v) = h(S_{i - 1} \cup \{v\}) - f(S_{i - 1}). \]
          \item Pick $u_i \in \Domain$ from the probability distribution
                 \[ \vec{g} \left( q_i(v_1), \dots, q_i(v_d) \right). \]
          \item $S_i \leftarrow S_i \cup \{u_i\}$.
        \end{Enumerate}
  \item Return $S_k$.
\end{Enumerate}
\bigskip

\noindent To analyze Algorithm 1 we need the following result for compositions of
differentially private algorithms.

\paragr{Composition of Differentially Private Algorithms.} An algorithm $A$ is a composition of
$k$ algorithms $A_1, \dots, A_k$ if the output of $A(\vecv)$ is a function only of the outputs
$A_1(\vecv), \dots, A_k(\vecv)$.

The following theorem bounds the privacy of $A(\vecv)$ as a function of the privacy of
$A_1(\vecv), \dots, A_k(\vecv)$.

\begin{theorem}[\cite{DworkRV10}] \label{thm:compositionLemma}
    Let $A_1, \dots, A_k$ be differentially private algorithms with parameters
  $(\eps', \delta')$. Let also $A$ a composition of $A_1, \dots, A_k$. Then,
  $A$ satisfies $(\eps, \delta)$-differential privacy with
  \begin{Enumerate}
    \item $\eps = k \eps'$ and $\delta = k \delta'$,
    \item $\eps = \frac{1}{2} k^2 \eps'^2 + \sqrt{2\log(1/\eta)} \eps'$ and
          $\delta = \eta + k \delta'$ for any $\eta > 0$.
  \end{Enumerate}
\end{theorem}

  We are now ready to prove Theorem \ref{thm:monotoneSubmodularCardinality}.

\begin{proof}[Proof of Theorem \ref{thm:monotoneSubmodularCardinality}.]
    The privacy guarantee easily follows from the composition properties of
  differentially private mechanisms that we present in Theorem
  \ref{thm:compositionLemma}.

    Let $S^*$ be the set of the optimal solution, $S_i$ be the set that the
  algorithm has in the $i$th iteration and $v_i$ the $i$th element that our
  algorithm chose. We have that
  \begin{align*}
    \Exp[&h(S_i \cup \{v_i\}) - h(S_i)] = \\
    & = \frac{1}{1 + \delta} \max_{v \in \Domain \setminus S_{i - 1}} (h(S_i \cup \{v\}) - h(S_i)) \\
                                       & \ge \frac{1}{1 + \delta} \frac{1}{k} \left( \sum_{v \in S^*} (h(S_i \cup \{v\}) - h(S_i)) \right) \\
                                       & \ge \frac{1}{1 + \delta} \frac{1}{k} (h(S^* \cup S_{i - 1}) - h(S_{i - 1})) \\
                                       & \ge \frac{1}{1 + \delta} \frac{1}{k} (\OPT - h(S_{i - 1})).
  \end{align*}
  \noindent Therefore
  \[ \OPT - \Exp[h(S_i)] \le \left( 1 - \frac{1}{1 + \delta} \frac{1}{k} \right)^i \OPT. \]

  \noindent From which we conclude
  \begin{align*}
    \Exp[h(S_k)] & \ge \left(1 - \left(1 - \frac{1}{1 + \delta} \frac{1}{k}\right)^k\right) \OPT \\
                 & \ge \left(1 - \frac{1}{\exp(1/(1 + \delta))} \right) \OPT.
  \end{align*}
  and hence the theorem follows.
\end{proof}

  Next our goal is to compare Theorem 8 of \cite{MitrovicBKK17} with Theorem
\ref{thm:monotoneSubmodularCardinality}. We illustrate the difference between power and
exponential mechanism showing an improvement over the state of the art algorithm of
\cite{MitrovicBKK17}.

\begin{lemma} \label{lem:boundSubmodularSymmetricRenyiFromLInfinity1}
  Let $\delta_{\Pow}$ be the approximation loss of $\Pow$
  assuming that the input data set is $t$-multiplicative insensitive, then
  \( \delta_{\Pow} \le \min \left\{ \frac{1}{e} + \frac{2 \sqrt{k} \log d}{t \eps} \frac{S_{\infty}(h)}{\OPT}, 1 \right\} \).
\end{lemma}

\ifappendix
\begin{proof}
  From Theorem \ref{thm:monotoneSubmodularCardinality} we have that
  \begin{align*}
    \delta_{\Pow} & = \min \left\{ \exp \left( - \left(1 - \frac{S_{\infty}(h)}{\OPT}\right)^{\frac{2 \sqrt{k} \log d}{\eps}} \right), 1 \right\} \\
                  & \le \min \left\{ \exp \left( - \left(1 - \frac{2 \sqrt{k} \log d}{\eps} \frac{S_{\infty}(h)}{\OPT}\right) \right), 1 \right\} \\
                  & = \frac{1}{e} \min \left\{ \exp \left( \frac{2 \sqrt{k} \log d}{\eps} \frac{S_{\infty}(h)}{\OPT} \right), e \right\}
  \end{align*}
  \noindent Now if
  $\frac{2 \sqrt{k} \log d}{\eps} \frac{S_{\infty}(h)}{\OPT} \ge 1$ then
  $\delta_{\Pow} = 1$ and
  hence, we can assume that
  $\frac{2 \sqrt{k} \log d}{\eps} \frac{S_{\infty}(h)}{\OPT} \le 1$. But for
  any $z \le 1$ it is easy to see that $e^z \le 1 + e z$ and hence
  \begin{align*}
    \delta_{\Pow} & \le \frac{1}{e} \min \left\{ 1 + e \frac{2 \sqrt{k} \log d}{\eps} \frac{S_{\infty}(h)}{\OPT}, e \right\}
  \end{align*}
  \noindent and the lemma follows.
\end{proof}
\else
The proof is available in the Supplementary Material.
\fi
Now combining Theorem \ref{thm:monotoneSubmodularCardinality} and Lemma
\ref{lem:boundSubmodularSymmetricRenyiFromLInfinity1} we can prove Corollary 
\ref{cor:monotoneSubmodularCardinalityWithAssumption} which clearly illustrates the comparison of
the performance of power and exponential mechanism. From Corollary 
\ref{cor:monotoneSubmodularCardinalityWithAssumption} we observe that the approximation loss using
the exponential mechanism is a $O(\sqrt{k})$ factor larger than the approximation loss using the
power mechanism. Hence Corollary \ref{cor:monotoneSubmodularCardinalityWithAssumption}
improves over the state of the art differentially private algorithms for submodular optimization.

  We can use the same ideas as in Theorem
\ref{thm:monotoneSubmodularCardinality} and Corollary
\ref{cor:monotoneSubmodularCardinalityWithAssumption} to improve the results for maximization of
submodular functions with more general matroid constraints of \cite{MitrovicBKK17}.

\section{Experiments on Large Real-World Data Sets} \label{sec:experiments}

\hrule
\emph{
\paragr{Remark.} In the main part we accidentally refer to Appendix \ref{sec:experiments} both
for the theoretical and the practical results about differentially private submodular 
maximization. Please look at the Appendix \ref{sec:submodular} for the details on the theoretical
part and in this section for the details in the experiments part.
}
\smallskip
\hrule
\medskip

\begin{figure}[t]
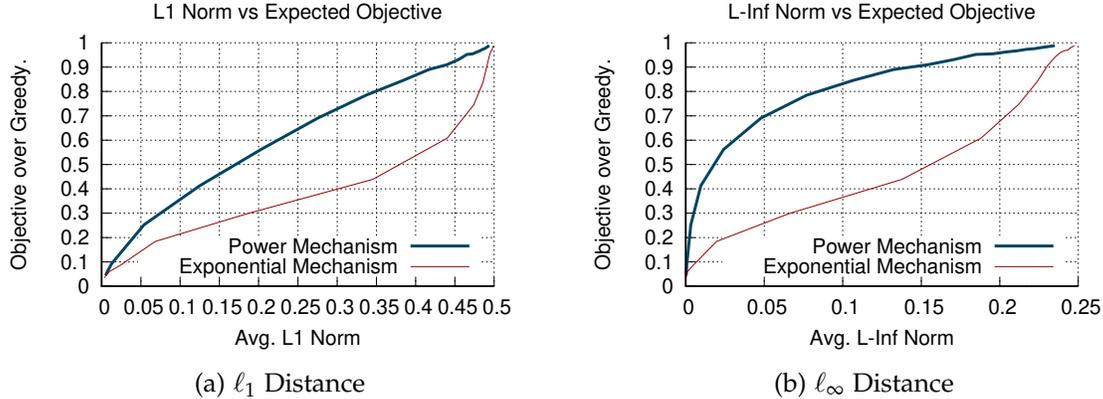

\centering
  \begin{subfigure}[t]{0.45\textwidth}
    \includegraphics[width=0.95\textwidth]{img/dblp-l1.eps}
    \caption{$\ell_1$ Distance}
    \label{fig:experiment-1-sub-appendix}
  \end{subfigure} ~
  \begin{subfigure}[t]{0.45\textwidth}
      \includegraphics[width=0.95\textwidth]{img/dblp-l0.eps}
      \caption{$\ell_\infty$ Distance}
      \label{fig:experiment-2-sub-appendix}
  \end{subfigure}
\caption{Robustness vs objective value in the submodular maximization with cardinality constraint $k=10$. The y-axis shows the ration of the average objective obtained vs the (non-private) greedy algorithm. The x-axis represent the sensitivity to the manipulation test of the value of the first element selected. }\label{fig:experiment-sub-appendix}
\end{figure}

We now empirically validate our results for submodular maximization. In our experiments we used a publicly
available data-set to create a max-k-coverage instance similarly
to prior work~\cite{SPAA17}. In a coverage instance we are given a family $N$ of
sets over a ground set $U$ and  we want to find $k$ sets from $N$ with maximum
size of their union (which is a monotone submodular maximization problem under
cardinality constraint). We created the coverage instance from the {\bf DBLP
co-authorship} network of computer scientists by extracting, for each author,
the set of her coauthors. The ground set is the set of all authors in DBLP.
There are $\sim 300$ thousands sets over $\sim 300$ thousands elements for a
total sum of sizes of all sets of $1.0$ million. Then we ran the (non-private)
greedy submodular maximization algorithm to obtain a (baseline) upperbound on
the solution (notice that computing the actual optimum is NP-Hard). Then we
compared the objective value obtained by private greedy algorithm for submodular
maximization using the exponential mechanism (as described in Algorithm 1
in~\cite{MitrovicBKK17}) and using the power mechanism as soft-max, for different
values of the parameter $\alpha$ in the two methods. We used $k=10$ as the
cardinality of the output in our experiment. 

To evaluate empirically the smoothness of the mechanism we performed a manipulation test on the data. We manipulated the coverage instance removing, independently, each element of the ground set with probability $1/1000$. Then, for a fixed mechanism and parameter setting, we compared the probability distribution of 
the first set selected by the algorithm in the manipulated
instance vs in the original instance (we used the $\ell_1$ and $\ell_\infty$ distance
of the distributions)\footnote{Ideally one would like to compare the
distribution of the output value of the algorithm for the actual $k$. However,
computing or even approximating well the distribution of value of the output is
computationally hard, so we resort to computing exactly the distribution of the
first item selected.}. Finally, we ran each configuration of the experiment (i.e., a mechanism and a parameter) $100$ times and reported the average objective in the original dataset (over the objective of
non-private greedy) and average distance between the distributions obtained over the original and manipulated datasets.
Figures~\ref{fig:experiment-1-sub-appendix} and~\ref{fig:experiment-2-sub-appendix} report the
results for $k=10$ in  the DBLP instance. Notice that we observe that for the same level of sensitivity to manipulation (both
in $l_1$ and $l_\infty$ norm) the power mechanism obtains significantly more
objective value in this problem as well (y-axis reports the average ratio of the
objective obtained vs that of the non-private algorithm). This confirms our
theoretical results for submodular maximization.

\section{Loss Function For Multi-class Classification} \label{sec:lossFunction}

  Before presenting our loss function that can be used for multi-class classification we present
a proof of Lemma \ref{lem:sparsemaxLowerBound}. Due to a minor typo in the presentation of the
Lemma in the main part of the paper we restate the Lemma here corrected.

\begin{lemma}[Lemma \ref{lem:sparsemaxLowerBound}] \label{lem:sparsemaxLowerBound:corrected}
    Let $h(\cdot) = \mathrm{sparsegen}\text{-}\mathrm{lin}(\cdot)$ be the generalization of
  $\mathrm{sparsemax}(\cdot)$ function, then there exist $\vecx, \vecy \in \R^d$ such that 
  $\norm{h(\vecx) - h(\vecy)}_1 \ge \frac{1}{2} d^{1 - 1/q} \norm{\vecx - \vecy}_q$.
\end{lemma}

\begin{proof}[Proof of Lemma \ref{lem:sparsemaxLowerBound}.]
    We set $\vecx = \vec{0}$ and $\vec{y}$ such that $y_i = 2/d$ for $i \le d/2$ and $y_i = 0$ 
  otherwise. Doing simple calculations we get that $h(\vecx) = (1/d) \cdot \vec{1}$, whereas
  $h_i(\vecy) = 2/d$ for $i \le d/2$ and $h_i(\vecy) = 0$ otherwise. Hence we have 
  $\norm{h(\vecx) - h(\vecy)}_1 = 1$ and
  \[ \norm{\vecx - \vecy}_q = (2/d)^{1 - 1/q} \le 2 / d^{1 - 1/q} \]
  and the lemma follows.
\end{proof}

  In this section, we show how our mechanism can be used in multi-class 
classification by proposing the corresponding loss function.

  First, we note that the 
$\calL_{\mathrm{sparsegen}\text{-}\mathrm{lin}, \mathrm{hinge}}$ loss function defined in
\cite{MartinsA16} can be used as a loss function for any soft-max function that satisfies: (1) 
permutation invariance, (2) $\delta$-worst-case approximation additive loss, where we have to set
$\delta = 1 - \lambda$. The main issue of this loss function is that it does not take into account
specific structural properties of the soft-max function used. For this reason, we propose an  
alternative loss function.

A loss function that corresponds to \plsm\ with parameter $\delta$ is a function 
$L : \R^d \times \Delta_d \to \R_+$ such that for any $\vecx \in \R^d$ and $\vecq \in \Delta_d$, it
holds that $L(\vecx; \vecq) = 0 \Leftrightarrow \Rec^{\delta}(\vecx) = \vecq$. Our loss function has
three components: (1) $L_{ord}$ is minimized only when the ordering of $\vecx$ is the same as the
ordering of $\vecq$, (2) $L_{supp}$ is minimized when the coordinates of $\vecx$ that are within
$\delta$ from $\norm{\vecx}_{\infty}$ correspond to the coordinates $i$ such that $q_i > 0$, and (3)
$L_{sqr}$ minimizes the error between $\Rec^{\delta}(\vecx)$ and $\vecq$ assuming they have the same
order. Finally, our loss function $L_{\Rec}$ is the sum of these three components, i.e. 
$L_{\Rec} = L_{ord} + L_{supp} + L_{sqr}$.
\smallskip

\paragr{Order Regularization.} For every $\vecq \in \Delta_d$, let $\pi_{\vecq}$ be the 
permutation of the coordinates $[d]$ such that 
$q_{\pi_{\vecq}(1)} \ge \cdots \ge q_{\pi_{\vecq}(d)}$, then 
\[ L_{ord}(\vecx; \vecq) = \sum_{i = 1}^{d - 1} \max\{x_{\pi_{\vecq}(i + 1)} - x_{\pi_{\vecq}(i + 1)}, 0\}. \]

\paragr{Support Regularization.} Let $\vecq \in \Delta_d$, let $S \subseteq [d]$ be the subset of
the coordinates $[d]$ such that $i \in S \Leftrightarrow q_i > 0$, let also $\delta$ be the 
parameter of \plsm, then
\[ L_{supp}(\vecx; \vecq) = \sum_{i \in S} \max\{x_{\pi_{\vecq}(1)} - x_i - \delta, 0\} + \sum_{i \in [d] \setminus S} \max\{x_i - x_{\pi_{\vecq}(1)} + \delta, 0\}. \]

\paragr{Square Loss.} Let $\vecq \in \Delta_{d - 1}$, then
\[ L_{sqr}(\vecx; \vecq) = \norm{\vecq - \frac{1}{\delta} \matP_{\pi_{\vecq}}^{-1} \matr{SM}_{(k_{\vecq} ,d)} \matP_{\pi_{\vecq}} \vecx - \matP_{\pi_{\vecq}}^{-1} \vec{u}^{(k_{\vecq})}}_2^2. \]

  The main properties of the loss function $L_{\Rec}$ are summarized in Proposition 
\ref{prop:lossFunction:properties}. This proposition suggests that $L_{\Rec}$ can be used as a
meaningful loss function in multiclass classification. 

\begin{proof}[Proof of Proposition \ref{prop:lossFunction:properties}.]
    The property (1) follows directly from the fact that $L_{\plsm}$ is a sum of non-negative 
  terms. Also observe that: (i) $L_{ord} = 0$ if and only if the order of the coordinates of the 
  vector $\vecx$ agrees with the order of the coordinates of $\vecq$, and (ii) $L_{supp} = 0$ if
  and only if the only coordinates that are $\delta$-close to $\norm{\vecx}_{\infty}$ are the 
  coordinates for which $q_i > 0$. Using (i) and (ii) together with $L_{sqr} = 0$ we can see that
  the property (2) of Proposition \ref{prop:lossFunction:properties} is implied. Property (3)
  follows again easily from the fact that the maximum of two convex function is convex and the 
  sum of convex functions is also convex.
\end{proof}

\end{document}